\documentclass[10pt]{article} %
\usepackage[accepted]{tmlr}

\usepackage{hyperref}
\usepackage{url}

\usepackage{amsmath}
\usepackage{amssymb}
\usepackage{mathtools}
\usepackage{amsthm}

\usepackage{makecell} 

\usepackage[capitalize]{cleveref}

\usepackage[textsize=tiny]{todonotes}
\usepackage{graphicx}
\usepackage[T1]{fontenc}
\usepackage{amsmath,amssymb}
\usepackage{bbm}
\usepackage{float}
\usepackage{accents}
\usepackage{hyperref}
\usepackage{algorithm}
\usepackage{algorithmic}
\usepackage{multicol}
\usepackage{wrapfig}
\usepackage{enumitem}
\usepackage{booktabs}
\usepackage{multirow}
\usepackage{listings}
\usepackage{longtable}
\usepackage{tabularx}

\newcommand{\ouralg}{DP-GRAPE}

\usepackage{caption}
\usepackage{subcaption}
\usepackage{xcolor}

\newcommand{\xmark}{$\times$}

  \providecommand{\bb}{\mathbf{b}}

  \let\lll\ll
  \renewcommand{\ll}{\mathbf{l}}

  \providecommand{\mI}{\mathbf{I}}

  \providecommand{\cN}{\mathcal{N}}

\newcommand{\R}{\mathbb{R}}
\newcommand{\abs}[1]{\left\lvert #1\right\rvert}
\newcommand{\norm}[1]{\left\lVert #1\right\rVert}

\DeclareMathOperator{\E}{\mathbb{E}}

\newcommand{\clip}[1]{\mathrm{clip}\left( #1\right)}

\def\remark{\addtocounter{remark}{1}\def\@currentlabel{\theremark}%
\emph{Remark~\theremark}. } \makeatother
\newcounter{remark}
 \usepackage[]{color-edits}
 \addauthor{xw}{magenta}
 \addauthor{dev}{blue}

\renewcommand{\c}[1]{\ensuremath{\mathcal{#1}}} %

\usepackage{tikz}
\makeatletter
\newcommand*{\encircled}[1]{\relax\ifmmode\mathpalette\@encircled@math{#1}\else\@encircled{#1}\fi}
\newcommand*{\@encircled@math}[2]{\@encircled{$\m@th#1#2$}}
\newcommand*{\@encircled}[1]{%
  \tikz[baseline,anchor=base]{\node[draw,circle,outer sep=0pt,inner sep=.2ex] {#1};}}
\makeatother

\newcommand{\eps}{\varepsilon}

\newcommand{\paren}[1]{\left( #1\right)}

\newcommand{\smo}{\lambda}
\newcommand{\lip}{\Gamma}
\newcommand{\indset}{\c U}

\newcommand{\ind}{\ensuremath{\mathbbm{1}}} %

\renewcommand{\clip}{\mathrm{clip}}

\renewcommand{\bb}[1]{\mathbb{#1}}

\expandafter\def\expandafter\normalsize\expandafter{%
    \normalsize%
    \setlength\abovedisplayskip{4pt}%
    \setlength\belowdisplayskip{4pt}%
}

\theoremstyle{plain}
\newtheorem{theorem}{Theorem}[section]

\newtheorem{lemma}[theorem]{Lemma}

\theoremstyle{definition}

\newtheorem{assumption}[theorem]{Assumption}
\theoremstyle{remark}

\usepackage{savesym}
\savesymbol{AND} %
\usepackage{algorithmic}

\title{Memory-Efficient Differentially Private Training with \\Gradient Random Projection}

\author{\name Alex Mulrooney \email alexmul@udel.edu \\
      \addr University of Delaware\\
      \AND
      \name Devansh Gupta \email guptadev@usc.edu \\
      \addr University of Southern California\\
      \AND
      \name James Flemings \email jamesf17@usc.edu \\
      \addr University of Southern California\\
      \AND
      \name Huanyu Zhang \email zhanghuanyu99@gmail.com \\
      \addr Meta\\
      \AND
      \name Murali Annavaram \email annavara@usc.edu \\
      \addr University of Southern California\\
      \AND
      \name Meisam Razaviyayn \email razaviya@usc.edu \\
      \addr University of Southern California\\
      \AND
      \name Xinwei Zhang \email xinweizh@amazon.com \\
      \addr Amazon\\
}

\def\month{05}  %
\def\year{2026} %
\def\openreview{\url{https://openreview.net/forum?id=CyxgbXCrWZ}} %

\begin{document}

\maketitle

\begin{abstract}
Differential privacy (DP) protects sensitive data during neural network training, but standard methods like DP-Adam suffer from high memory overhead due to per-sample gradient clipping, limiting scalability. We introduce \ouralg{} (Gradient RAndom ProjEction), a DP training method that significantly reduces memory usage while maintaining utility on par with first-order DP approaches. \ouralg{} is motivated by our finding that privatization flattens the gradient singular value spectrum, making SVD-based projections (as in GaLore \citep{zhao2024galore}) unnecessary. Consequently, \ouralg{} employs three key components: (1) random Gaussian matrices replace SVD-based subspaces, (2) gradients are privatized after projection, and (3) projection is applied during backpropagation. These contributions eliminate the need for costly SVD computations, enable substantial memory savings, and lead to improved utility. Despite operating in lower-dimensional subspaces, our theoretical analysis shows that \ouralg{} achieves a privacy-utility trade-off comparable to DP-SGD. Our extensive empirical experiments show that \ouralg{} can significantly reduce the memory footprint of DP training without sacrificing accuracy or training time. In particular,  \ouralg{} reduces memory usage by over 63\% when pre-training Vision Transformers and over 70\% when fine-tuning RoBERTa-Large as compared to DP-Adam, while achieving similar performance. We further demonstrate that \ouralg{} scales to fine-tuning large models such as OPT with up to 6.7 billion parameters, a scale at which DP-Adam fails due to memory constraints. Our code is available at \url{https://github.com/alexmul1114/DP_GRAPE}.
\end{abstract}

\section{Introduction} \label{sec:introduction}
While deep neural networks have shown exceptional performance in Computer Vision and Natural Language Processing, they can be vulnerable to attacks that reveal training instances \citep{fredrikson2015model, carlini2021extracting, mireshghallah2022memorization}. This poses risks when training on sensitive data, such as personal or medical records. Differential privacy (DP) offers a principled way to protect the privacy of individual samples \citep{dwork2006calibrating, dwork2014algorithmic}. Methods such as DP-SGD and DP-Adam enforce DP during neural network training by clipping per-sample gradients and adding Gaussian noise calibrated to the desired privacy level \citep{abadi2016deep}.

{
\renewcommand{\arraystretch}{0.90}
\begin{table*}[b]
\caption{\small Comparison of benefits of different memory-saving methods for DP training. We consider a method ``high-utility'' if it can achieve close to the utility of DP-Adam, see \cref{sec:exp-roberta-finetuning}.}
\begin{center}
\begin{small}
\setlength{\tabcolsep}{2.5pt}
\begin{tabular}{l|ccccc}
\toprule
Method & High-Utility & Pre-training & Extra Computation & \makecell{Low-Memory\\Optimizer States} & \makecell{Total Number \\ of Steps to Converge}  \\
\midrule
\citet{bu2021fast} & \xmark & \checkmark & JL Projections & No & Medium \\
Ghost Clipping& \checkmark & \checkmark & Extra Backward Pass & No & Medium \\
Book-Keeping & \checkmark & \checkmark & Extra Partial Gradients & No & Medium \\
Zeroth-order DP & \xmark & \xmark & - & Yes & Large \\
DP-LoRA & \xmark & \xmark & - & Yes & Medium \\
\ouralg{} & \checkmark & \checkmark & Low-overhead Random Projection & Yes & Medium \\
\bottomrule
\end{tabular}
\end{small}
\end{center}
\label{tab:alg-comp}
\end{table*}
}

In addition to privacy, training large models efficiently requires careful memory management. Standard optimizers such as Adam demand substantial memory to store parameters, activations, gradients, and optimizer states—posing a challenge for models with billions of parameters, especially on limited-memory GPUs. To address this, recent work has explored low-rank training methods in non-private settings. LoRA reduces optimizer memory by constraining weight updates to low-rank matrices \citep{hu2021lora}, while GaLore projects layer gradients onto SVD-derived subspaces \citep{zhao2024galore}. To avoid costly SVD computations at scale, alternatives such as randomized SVD \citep{pasand2024rgp}, Gaussian random projections \citep{hao2024flora}, and Stiefel manifold methods \citep{he2024subspace} have been proposed. 

When both privacy and memory constraints are present, training becomes even more challenging. DP-SGD and DP-Adam introduce significant overhead by requiring per-sample gradient clipping, with memory usage scaling linearly with model and batch size—often becoming the dominant cost. Ghost clipping reduces this cost by using an extra backward pass to compute sample gradient norms without instantiating full gradients \citep{lee2021scaling, li2021large}, but it adds compute overhead and does not reduce optimizer state memory. Book-Keeping \citep{bu2023differentially} removes the extra pass using additional tricks, but increases memory usage by storing partial gradients and still offers no savings for optimizer states. As we will see, both of these  methods are orthogonal to our approach and can be combined for further memory reduction with our method.

Other strategies have also been explored: Approximating gradient norms via random projections saves memory but yields poor privacy guarantees at low projection dimensions \citep{bu2021fast}. DP-LoRA combines LoRA with differential privacy to reduce trainable parameters, but it is not applicable for pre-training and falls short of DP-Adam in fine-tuning \citep{yu2021differentially} (also see \cref{tab:roberta-fewshot-results}). Zeroth-order methods such as DP-ZO \citep{tang2024private} and DPZero \citep{zhang2024dpzero} are memory-efficient but generally underperform first-order DP methods in NLP fine-tuning and cannot be used for pre-training. In addition, zeroth-order methods require a large number of steps to converge even for non-private training \citep{malladi2023fine,li2024addax}. We summarize the trade-offs of these methods in \cref{tab:alg-comp}, including utility, pre-training compatibility, added computation, and optimizer memory usage.

Another related line of work projects gradients onto subspaces obtained from auxiliary non-sensitive data \citep{yu2021not, zhou2020bypassing, gu2023choosing}. While this approach can improve model utility, it requires having non-sensitive data available, and previous works have not exploited the potential memory benefits of gradient projection to enable scaling to larger models.

To enable memory-efficient DP training for large models without sacrificing utility, we propose \ouralg{}. \ouralg{} projects sample gradients onto lower-dimensional subspaces using random Gaussian matrices during backpropagation, reducing memory usage for both sample gradients and optimizer states. Our main contributions are:

\begin{itemize}[leftmargin=*, nosep]
    \item We demonstrate that privatization flattens the singular value spectrum of gradients, motivating the use of random projections (as in \citep{hao2024flora}) rather than SVD-based projections (as in \citep{zhao2024galore}).
    \item We establish that projecting gradients before privatization is critical for achieving high utility and memory efficiency. \ouralg{} uses the principle of privatizing the {\it projected gradients} in lower dimensional space to achieve significant memory efficiency gains and improved utility over a na\"ive approach that privatizes gradients before projection. We primarily obtain significant memory savings in storing the optimizer states, since we only need to maintain the projected states. This leads to substantially lower memory usage compared to differentially private implementations of Adam, which must store full first and second moment estimates for all parameters.
    \item We provide a set of novel analyses of the privacy-utility guarantee for the proposed \ouralg{} with projection matrices having random but unbounded entries. The theoretical result indicates that \ouralg{} enjoys the same utility guarantee as DP-SGD.
    \item Our experiments demonstrate the efficiency and high utility of the proposed algorithm. In particular, compared to DP-Adam, our algorithm achieves similar performance with significantly less memory usage (e.g., $24.4$GB compared with $78.1$GB when training RoBERTa-Large). Compared to recent memory-efficient DP training methods such as DP-ZO and DP-LoRA, our algorithm can be used for pre-training, while the others are restricted to fine-tuning tasks only. Furthermore, \ouralg{} requires significantly fewer iterations to converge than zeroth-order DP methods, reducing total training time by a factor of more than 6.
\end{itemize}

Concurrently with our work, D2P2-SGD \citep{jiangd2p2} introduced random gradient projection in the DP setting, with a focus on tighter theoretical error bounds rather than memory efficiency. A key distinction is that \ouralg{} projects gradients before privatization, enabling substantial memory savings and improved practical utility, as demonstrated in \cref{sec:exp-vit-pretraining}.

Our work is most closely related to other subspace methods such as LoRA \citep{hu2021lora} and GaLore \citep{zhao2024galore}. LoRA changes the model architecture by adding trainable low-rank matrices while keeping the original weights frozen. Its DP version, DP-LoRA \citep{yu2021differentially}, applies DP-SGD to these low-rank parameters, keeping the underlying optimization algorithm the same and obtaining privacy guarantees directly from DP-SGD. While this can work well for fine-tuning, the low-rank assumption is too restrictive for pre-training, as opposed to \ouralg{}. Furthermore, we find that \ouralg{} achieves better performance than DP-LoRa across different fine-tuning tasks with RoBERTa-Large. As shown in \cref{tab:roberta-fewshot-results}, \ouralg{} outperforms DP-LoRA on $11$ of the $12$ total combinations of privacy levels and tasks, often by a significant margin. We also compare \ouralg{} with a na\"ive DP-GaLore algorithm
in \cref{sec:methodology}, where we observe that simply adding DP noise and clipping to GaLore eliminates much of its potential benefit: the clipping and noise flatten the gradient spectrum, eliminating low-rank structure. Indeed, the na\"ive DP-GaLore algorithm performs much worse than DP-Adam and \ouralg{} in our Vision Transformer pre-training experiments, as shown in \cref{fig:vit_results}. Instead, we show how using random projections and privatizing the projected gradients confers significant computational and memory improvements.

\section{Background}

\paragraph{Notations \& Definitions} Throughout the paper, we consider minimizing a loss function $f$ over a dataset $X$ of size $n$ with samples $\{\xi_1, \dots, \xi_n \}$. We assume that $f$ is parameterized as a model with $L$ layers, with the $\ell^\mathrm{th}$ layer holding the matrix $W_{\ell} \in \mathbb{R}^{m_{\ell} \times n_{\ell}}$ of size $m_{\ell} \times n_{\ell}$ as its trainable parameters. The total number of parameters is $d = \sum_{\ell=1}^L m_{\ell}n_{\ell}$. The problem we seek to solve is:
\begin{equation}
    \min_{ \{ W_{\ell} \}_{\ell=1}^L } \frac{1}{n}\sum_{i=1}^n f(\{ W_{\ell} \}_{\ell=1}^L ; \xi_i)
\end{equation}
Without loss of generality, we assume $m_{\ell} \leq n_{\ell}$ for all layers. When training the model with iterative methods (e.g., DP-SGD), we denote the total number of training steps as $T$ and index the steps with $(\cdot)^t$, the batch size as $B$ , and index the samples with $(\cdot)_i$.

Let $G_{\ell,i}^t = \nabla_{W_\ell^t} f(\{ W_{\ell}^t \}_{\ell=1}^L ; \xi_i)$ be the gradient for sample $i$ at iteration $t$ for layer $\ell$, and $\{ G_{\ell,i}^t \}_{i=1}^B$ be the collection of all sample gradients in the batch. We denote the concatenated gradients of all layers at iteration $t$ for sample $i$ as $G_i^t = \begin{bmatrix}
    \text{vec}(G_{1,i}^t)^{\top} & \cdots & \text{vec}(G_{L,i}^t)^{\top}
\end{bmatrix}^{\top}$, where $\text{vec}(G_{\ell,i}^t)$ is the vectorized sample gradient for sample $i$ at layer $\ell$ (a column vector of length $m_{\ell}n_{\ell}$). The clipping operation used to bound the norm of per-sample gradients is defined as 
$\text{clip}(G_i^t, C) = \min(1, \frac{C}{\lVert G_i^t \rVert_2}) G_i^t$ 
for sample gradient $G_i^t$ and clipping threshold $C > 0$. 

For methods that project gradients, we denote the projection matrix for layer $\ell$ as $P_{\ell} \in \R^{m_{\ell} \times r}$, where $r$ is the projection dimension. 
Let $\mathcal{N}_{s_{\ell}} (0, \frac{1}{r}) \in \mathbb{R}^{m_{\ell} \times r}$ be a matrix with entries i.i.d. from a Gaussian distribution with mean $0$ and variance $\frac{1}{r}$, generated using seed $s_{\ell}$. 
When drawing from a normal distribution without a particular seed (such as when adding noise to gradients), we omit the seed. We denote the projected gradient as $R = P^T G$, and the privatized one as $\tilde{R}$.

\paragraph{Differential Privacy}
Differential privacy ensures that an algorithm’s output does not change significantly when a single training sample is removed, protecting individual samples. Formally, for neighboring datasets $X$ and $X'$ differing by one sample, a randomized algorithm $\mathcal{A}: D \rightarrow O$, where $D$ is the set of all possible datasets and $O$ is the set of all possible outcomes, is $(\eps, \delta)$-DP if \citep{dwork2006our}:
\begin{equation}
    P(\mathcal{A}(X) \in O) \leq e^{\eps}P(\mathcal{A}(X') \in O) + \delta.
\end{equation} 
A common way to make a function $h: X \rightarrow \mathbb{R}^d$ differentially private is to add Gaussian noise scaled to the function's $\ell_2$ sensitivity, which is defined as
$
    \Delta_2 h := \sup_{X, X'} \lVert h(X) - h(X') \rVert_2:
$

\begin{theorem}[\citet{dwork2014algorithmic}]
    Given a function $h: X \mapsto \mathbb{R}^d$ with $\ell_2$ sensitivity $\Delta_2h$, a dataset $X$, and $\eps, \delta > 0$, the randomized algorithm $\mathcal{A}(X) = h(X) + z$, where $z \sim \mathcal{N}\paren{0, \frac{\Delta_2 h}{\eps} \sqrt{2 \log \paren{\frac{1.25}{\delta}}}\mI_d}$, is $(\eps, \delta)$-DP. 
    \label{th:dp}
\end{theorem}

\paragraph{Differentially Private Optimization}

To train a differentially private neural network, noise is added to the gradients of each sample during training, rather than to the outputs. Since the gradients' $\ell_2$ sensitivity is often unbounded, they are clipped to a constant $C$ to limit the maximum norm. This gradient clipping, followed by the addition of noise calibrated to the desired privacy level, can be applied to any gradient-based optimization method such as SGD or Adam to achieve $(\varepsilon, \delta)$-DP \citep{abadi2016deep}. See \cref{app:dpadam} for the full DP-Adam algorithm in our notation. However, gradient privatization can reduce model utility. Additionally, clipping per-sample gradients requires computing individual gradients for each sample (rather than just the gradient averaged over all samples in the batch as in non-private training), increasing gradient memory usage from $d$ to $Bd$. For large models, this forces either a smaller batch size (which may reduce utility) or more gradient accumulation steps (which significantly increases the total training time).

\paragraph{Memory-Efficient Training with Gradient Projection}
One approach to reduce memory usage of training is to project gradients onto lower-dimensional subspaces, allowing optimizer states (e.g., the moment estimates in Adam) to be stored in the subspace. This reduces the memory usage of optimizer states from $\sum_{\ell} m_{\ell}n_{\ell}$ to $r\sum_{\ell} n_{\ell}$, saving significant memory when $r \lll m_{\ell}$. Various projection methods have been proposed, including using the SVD of layer gradients and random matrices. In GaLore \citep{zhao2024galore}, gradients are projected onto subspaces spanned by the top $r$ left singular vectors from the SVD of the layer gradient at a previous iteration, i.e., the projection matrix is $P_{\ell}^t = U[:,:r]$. FLoRA \citep{hao2024flora} instead uses a Gaussian matrix for $P_{\ell}^t$, i.e., $(P_{\ell}^t)_{i,j} \sim \mathcal{N}(0,\frac{1}{r})$. Alternatively, \citet{he2024subspace} use a uniform distribution on the Stiefel manifold to generate $P_{\ell}^t$.

\section{Methodology} \label{sec:methodology}
Our approach adapts memory-efficient training methods that use gradient projections such as GaLore~\citep{zhao2024galore} and Flora~\citep{hao2024flora} to the DP setting. Integrating gradient projection with DP introduces two important design choices: the type of projection (e.g., SVD-based or random), and the order of operations (privatizing gradients before or after projection). We systematically evaluate these choices. Since SVD-based projection derives subspaces from the gradients, to maintain privacy the gradients must be privatized before computing the SVD. However, as shown in \cref{fig:vit-svals}, privatization (i.e., clipping and noise addition) flattens the singular value spectrum, destroying any low-rank structure that the SVD aims to capture. This motivates the use of random projections, which can be applied before privatization and are computationally cheaper since they avoid computing an SVD for each layer. For an empirical comparison, we consider an algorithm that privatizes gradients before SVD-based projection, which we call na\"ive DP-GaLore due to its similarity to GaLore, and an algorithm that privatizes after random projection, which is \ouralg{}.

\subsection{Na\"ive DP-GaLore} \label{sec:naive-dpgalore}
Na\"ive DP-GaLore applies privatization and subsequently computes the SVD of layer gradients before projection (see \cref{app:galore-details} for details). However, this approach inherits the high memory cost of storing per-sample gradients of DP-SGD or DP-Adam, offering little improvement over them. In addition, computing SVDs for each layer becomes computationally expensive for large models. Finally, clipping full-dimensional gradients—rather than their projected versions—leads to degraded utility, as discussed in \cref{sec:exp-vit-pretraining}.

\subsection{Our Algorithm}
\ouralg{} uses random projection instead of the SVD and applies privatization of gradients after projection. This greatly reduces per-sample gradient memory and lead to significantly better utility (see \cref{sec:exp-vit-pretraining}) as compared to na\"ive DP-GaLore.

We describe our training method in \cref{algo:main}, with the projected Adam update in \cref{algo:adam}. For simplicity of presentation, we consider a model with $\ell$ linear layers, each parameterized by a matrix of size $m_{\ell} \times n_{\ell}$; nonlinear layers are handled as in DP-Adam (i.e., without projection). Each training iteration consists of a backward pass (steps 2–11), gradient privatization (steps 12–13), and optimizer and weight updates (step 14). To reduce memory, per-sample gradients are projected layer-by-layer during the backward pass (to avoid having the per-sample gradients for each layer instantiated at the same time). The projected gradients are then privatized and used for the projected Adam update.

Our use of random projections rather than the SVD is motivated by empirical observations about the effect of privatization on the singular values. As shown in \cref{fig:vit-svals}, while the non-private gradient exhibits low-rank structure, the combination of clipping and adding noise at levels needed to achieve typically-used levels of DP (e.g., $C=1.0$ and $\sigma=0.5,2.0$) flattens the spectrum of singular values, destroying the low-rankness. Furthermore, the use of random projections confers additional computational benefits by avoiding the computation of SVDs for the gradient matrix of each layer. From a memory prospective, as opposed to SVD subspaces, random projection matrices do not have to be stored for each layer since they can be cheaply generated from a random seed on-the-fly using \texttt{torch.randn} \citep{paszke2019pytorch} prior to the projection step.

\subsection{Memory Requirements}

{
\renewcommand{\arraystretch}{1.5}
\begin{table*}[t]
\caption{\small Memory usage of gradients, optimizer states, and projectors (if used) during training with batch size $B$ for different first-order non-DP and DP methods for an $L$-layer model of sizes $\{m_{\ell} \times n_{\ell}\}$. ``Gradient'' indicates the batch gradient for non-DP methods and all of the sample gradients for the batch for DP method. For GaLore, Na\"ive DP-GaLore, and \ouralg{}, $r$ denotes the projection dimension.}
\label{tab:memory-usage-comp}
\begin{center}
\begin{small}
\begin{tabular}{l|cccc}
\toprule
Method & Gradient & Optimizer States & Projectors \\
\midrule
Adam (non-private) & $\sum_{\ell=1}^L m_{\ell}n_{\ell}$ & $2\sum_{\ell=1}^L m_{\ell}n_{\ell}$ & - \\
GaLore (non-private) & $\sum_{\ell=1}^L m_{\ell}n_{\ell}$ & $2r\sum_{\ell=1}^L n_{\ell}$ & $r\sum_{\ell=1}^L m_{\ell}$ \\
DP-Adam & $B\sum_{\ell=1}^L m_{\ell}n_{\ell}$ & $2\sum_{\ell=1}^L m_{\ell}n_{\ell}$ & - \\
Na\"ive DP-GaLore & $B\sum_{\ell=1}^L m_{\ell}n_{\ell}$ & $2r\sum_{\ell=1}^L n_{\ell}$ & $r\sum_{\ell=1}^L m_{\ell}$ \\
\ouralg{} & $Br\sum_{\ell=1}^L n_{\ell}$ & $2r\sum_{\ell=1}^L n_{\ell}$ & $r \max\{ m_{\ell} \}_{\ell=1}^L$ \\
\bottomrule
\end{tabular}
\end{small}
\end{center}
\end{table*}
}

Compared to DP-Adam, \ouralg{} reduces memory usage by storing projected sample gradients and using lower-dimensional moment estimates in Adam. \Cref{tab:memory-usage-comp} compares the memory requirement of gradients, optimizer states, and projectors for non-private Adam and GaLore, DP-Adam, na\"ive DP-GaLore, and \ouralg{}.
Notably, \ouralg{} achieves the largest savings by reducing sample gradient memory from $B\sum_{\ell=1}^L m_{\ell}n_{\ell}$ (in DP-Adam and na\"ive DP-GaLore) to $Br\sum_{\ell=1}^L n_{\ell}$, since $r \lll m_{\ell}$. Similar to GaLore, it also reduces optimizer state memory from $2\sum_{\ell=1}^L m_{\ell}n_{\ell}$ to $2r\sum_{\ell=1}^L n_{\ell}$. Additionally, using random projections instead of SVD reduces projector memory to $r \max\{ m_{\ell} \}_{\ell=1}^L$, since only the random projector for a single layer needs to be loaded at one time (during the projection step for that layer).

\section{Analyzing the Privacy-Utility Tradeoff} \label{sec:theory}
We provide the privacy and convergence guarantee of our proposed algorithm below.

\begin{theorem}[Informal]
    Let $d = \sum_{\ell=1}^Lm_\ell n_\ell$. Given a $\Gamma$-lipschtiz and $\lambda$-smooth (potentially non-convex) objective function $f(\cdot; \xi): \bb R^d \rightarrow \bb R$ for all $\xi \in X,$ for any $0 < \eps \leq 2\ln(2/\delta)$ and $\delta\in(0,1)$, \cref{algo:main} is $(\eps, \delta)$-DP if $\sigma = \frac{2\sqrt{T\log(1/\delta)}}{n\epsilon}$. Moreover, there exists a set of hyper-parameters $\eta, B, C$ such that when $T = \frac{2\sqrt{2d} n \eps}{r\sqrt{\log(1/\delta)}}$, the output of~\cref{algo:main}, $\hat{W}$, for a simple SGD update rule satisfies 
    {    \small \begin{align*}
        \E \left[\norm{\nabla F(\hat{W})}^2\right] = \Tilde{\c O}\paren{\frac{\sqrt{d\log(1/\delta)}}{n\eps}},
    \end{align*}}%
    where the expectation is taken over all previous sampled batches, random matrices, the additive noise, and sampling of the final parameter vector.
    \label{th:informalconv}
\end{theorem}

\vskip 0.1in
\begin{algorithm}[h!]
\caption{\ouralg{}}
\label{algo:main}
\begin{algorithmic}[1]
\REQUIRE Dataset $X=\{\xi_1, \dots, \xi_n\}$, initial weights $\{W_{\ell}^0 \}_{\ell=1}^L$, learning rate $\eta$, subspace dimension $r$, subspace change frequency $F$, batch size $B$, clipping parameter $C$, noise level $\sigma$, total steps~$T$
\FOR {$t=1, 2, \dots, T$}
    \FOR{$\ell = L, L-1, \dots, 1$}
        \STATE $\{ G_{\ell,i}^t \}_{i=1}^B \gets \nabla_{W_{\ell}^t} f(\{ W_{\ell}^t \}_{\ell=1}^L ; \{ \xi_i \}_{i=1}^B)$ 
        \IF {$t \mod F = 0$} 
            \STATE Generate new $s_{\ell}^{t}$
        \ELSE
            \STATE $s_{\ell}^{t} \gets s_{\ell}^{t-1}$  
        \ENDIF
        \STATE $P_{\ell}^{t} \gets \mathcal{N}_{s_{\ell}^t}(0,\frac{1}{r}) \in \mathbb{R}^{m_{\ell} \times r}$
        \STATE $R_{\ell,i}^t \gets (P_{\ell}^{t})^{\top} G_{\ell,i}^t \; , \quad i=1,\dots,B$
    \ENDFOR 
    \STATE $R^t \gets  \frac{1}{B} \sum_{i=1}^B \text{clip} (R_i^t, C)$
    \STATE $\tilde{R}^t \gets R^t + \mathcal{N}(0, C^2\sigma^2 I) \in \mathbb{R}^{r\sum_{\ell=1}^L n_{\ell}}$
    \STATE {\small $\{ W_{\ell}^{t+1} \}^L_{\ell = 1} =\mathbf{Projected Adam Update}(\{ W_{\ell}^t \}^L_{\ell = 1}, \{ \tilde{R}_{\ell}^t \}^L_{\ell = 1}, \{ s_{\ell}^t \}_{\ell}^L, \eta)$}
\ENDFOR
\STATE Pick $\tau$ uniformly at random from $\{1, 2, \cdots, T\}$.
\STATE Return $\{ W_{\ell}^\tau \}_{\ell=1}^L$
\end{algorithmic}
\end{algorithm}

\vskip 0.25in
\begin{algorithm}[!h]
\caption{ProjectedAdamUpdate: Adam Update with Projected Moments}
\label{algo:adam}
\begin{algorithmic}%
\REQUIRE Model parameters $\{W_{\ell}^0 \}_{\ell=1}^L$, current projected first-order moments $\{M_{\ell}^{t-1} \in \mathbb{R}^{r \times n_{\ell}} \}_{\ell=1}^L$, current projected second-order moments $\{V_{\ell}^{t-1} \in \mathbb{R}^{r \times n_{\ell}}\}_{\ell=1}^L$, projector seeds $\{s_{\ell}^t\}_{\ell=1}^L$, step size $\eta$, decay rates $\beta_1, \beta_2$, iteration $t$, numerical stability constant $\phi$
\STATE $\alpha^t \gets \eta \frac{\sqrt{1 - \beta_2^t}}{1 - \beta_1^t}$ 
\FOR{$\ell = 1, 2, \dots, L$}  
    \STATE $M_{\ell}^t \gets \beta_1 M_{\ell}^{t-1} + (1-\beta_1)\tilde{R}_{\ell}^t$ %
    \STATE $V_{\ell}^t \gets \beta_2 V_{\ell}^{t-1} + (1-\beta_2)(\tilde{R}_{\ell}^t)^2$ %
    \STATE $P_{\ell}^{t} \gets \mathcal{N}_{s_{\ell}^t} \in \mathbb{R}^{m_{\ell} \times r}$
    \STATE $W_{\ell}^t \gets W_{\ell}^{t-1} - \alpha^t \cdot P_{\ell}^t(\frac{M_{\ell}^t}{\sqrt{V_{\ell}^t} + \phi})$
\ENDFOR
\STATE Return $\{W_{\ell}^t\}_{\ell=1}^L$, $\{M_{\ell}^t\}_{\ell=1}^L$, $\{V_{\ell}^t\}_{\ell=1}^L$
\end{algorithmic}
\end{algorithm}

\vskip 0.1in

\begin{figure}[b!]
    \centering
    \includegraphics[width=0.475\linewidth]{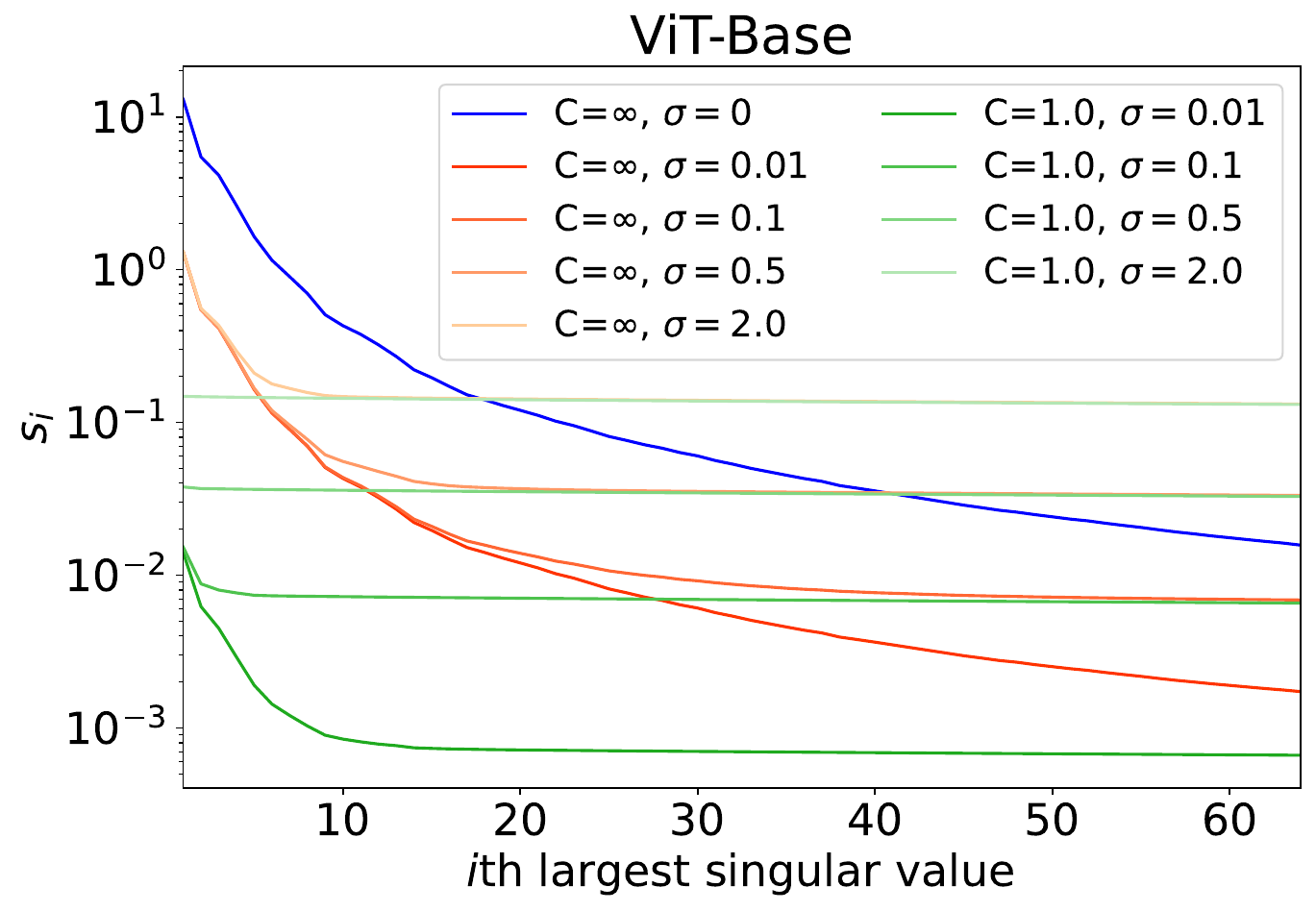}
    \includegraphics[width=0.475\linewidth]{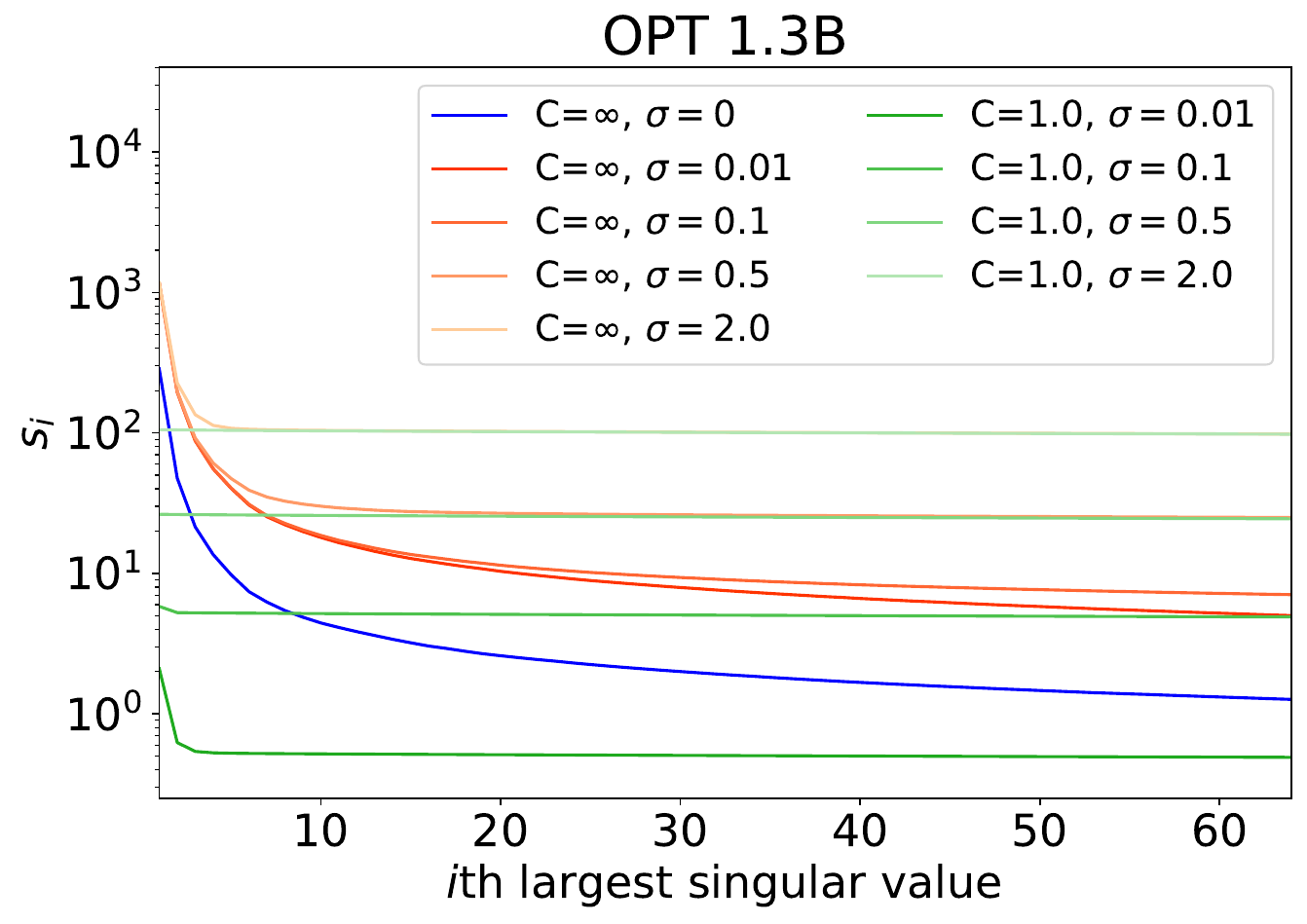}
    \caption{\footnotesize Left: singular values $s_i$ of layer gradient matrices with different clipping parameter $C$ and  noise levels $\sigma$, averaged across all layers of ViT-Base during training on CIFAR-10. Right: singular values of gradient matrices for OPT 1.3B during fine-tuning on SST-2. $C=\infty$ indicates no clipping. See \cref{app:experimental-details-vit} for details.}
    \label{fig:vit-svals}
\end{figure}
\normalsize

The formal statement of this theorem is Theorem~\ref{th:formalconv}, with assumptions detailed in~\cref{app:convergence}. We follow our result with some remarks that clarify the implications of our findings and their relevance to our setting.

\begin{remark}
    We recover the expected stationarity gap, i.e. the squared of the gradient norm at the returned iterate (up to log terms) of DP-SGD, $\c O\paren{\frac{\sqrt{d\log(1/\delta)}}{n\varepsilon}}$ \citep{wang2017differentially}.

\end{remark}

\begin{remark}
    In addition to the favorable scaling with respect to the number of layers, the dependence of the expected stationarity gap on the number of projections $r$ is minor and captured in the log terms. In fact, the log term decreases as $r$ increases up to a certain point (See~\cref{app:convergence}). However, the increase in random projections significantly lowers (by a factor of $r$) the time required for the algorithm to converge.
\end{remark}

These observations provide a broader context for our comparison with \citet{zhang2024dpzero}, which offers the first memory-efficient DP optimization guarantee. In the smooth case, their finite-difference step can be interpreted as a special case of our framework with rank-1 (\(r = 1\)) projection. Our analysis thus generalizes theirs to arbitrary \(r \geq 1\). However, their approach samples projection vectors uniformly from the unit sphere, requiring normalization of high-dimensional Gaussian vectors---an operation shown to be a bottleneck in DP optimization due to the cost of computing norms \citep{bu2021fast}. This overhead applies not only to gradients but also to random vectors used in projections, particularly when applied block-wise.

This exposes a key challenge: Gaussian-based projections have unbounded norm in the worst case, and clipping exacerbates this by increasing the expected norm of surviving samples. One potential solution involves truncating the Gaussian vector's components to lie within finite symmetric bounds (while adjusting for unbiasedness), ensuring convergence per the framework used in \citet{zhang2024dpzero}. However, this approach still incurs similar computational overheads as clipping - whether in terms of memory or runtime - due to the need to truncate each vector entry.

Through careful analysis, we demonstrate that sampling vectors with unbounded worst-case norms does not adversely affect convergence due to the exponential decay of tail probabilities, even in expectation. This insight allows us to exploit the efficient generation of standard normal vectors without requiring additional processing, thus addressing the computational challenges while maintaining theoretical guarantees.

\section{Experiments} \label{sec:experiments}
 We conduct experiments across three tasks to evaluate the performance of \ouralg{} and compare it against other DP methods: 1) \textbf{Pre-training} - training a VIT-base model from scratch on image classification tasks; 2) \textbf{Fine-tuning} - fine-tuning a RoBERTa-Large model on text classification tasks; 3) \textbf{Scalability} - demonstrating the scalability of \ouralg{} by successfully fine-tuning OPT models ranging from 1B to 6.7B parameters. Detailed experiment settings are provided in \cref{app:experimental-details}.

\subsection{Vision Transformer Training} \label{sec:exp-vit-pretraining}
To evaluate the effectiveness of \ouralg{} on pretraining tasks, we train Vision Transformer models (base model, 85M parameters) from scratch on MNIST \citep{deng2012mnist}, CIFAR10, and CIFAR100 \citep{krizhevsky2009learning}. To compare the performance and memory usage, we also train models using DP-Adam and na\"ive DP-GaLore, as discussed in \cref{sec:naive-dpgalore} (for pretraining, DP-LoRA and DP-Zero are not typically used, so we do not include comparisons). For all methods, we select the best clipping threshold and learning rate from a grid search, and then use those hyperparameters to train models for each method at $\eps=1,2,4,8$ privacy levels, with $\delta=\frac{1}{n}$. We do not use any additional public data or data augmentation. We discuss the experimental setup in detail in \cref{app:experimental-details-vit}.     

\textbf{Utility:} \Cref{fig:vit_results} shows both the final test accuracies for the different methods and datasets at varied privacy levels and the memory usage for each method across different batch sizes, with non-private Adam for comparison for the memory usage. See \cref{app:experimental-details-vit} for a full table of results. We find that while the na\"ive DP-GaLore approach performs significantly worse than DP-Adam, with an average decrease in accuracy across the different privacy levels of $12.1\%$ on MNIST, $7.8\%$ on CIFAR-10, and $4.1\%$ on CIFAR-100, \ouralg{} achieves comparable performance to DP-Adam. Averaged across privacy levels, \ouralg{} improves over DP-Adam by $1.3\%$ on MNIST, decreases by $2.5\%$ on CIFAR-10, and decreases by $0.6\%$ on CIFAR-100. While the accuracy of the models pre-trained with DP are relatively low, there are various strategies for improving performance, including data augmentation techniques \citep{de2022unlocking, bao2023dp} and training with limited public data \citep{bu2024pre}. However, in order to directly compare DP-Adam, the na\"ive DP-GaLore, and \ouralg{}, we do not integrate these techniques. In \cref{app:experimental-details-vit}, we also include results for fine-tuning on CIFAR-10 and CIFAR-100, where we find that \ouralg{} achieves similar accuracy as DP-Adam on CIFAR-10 and achieves significantly better accuracy on CIFAR-100.

\textbf{Memory Usage:} While achieving nearly the same accuracy as DP-Adam, \ouralg{} uses significantly less memory during training. When training on CIFAR-10 using a 24GB GPU, \ouralg{} allows a maximum batch size of around 165, while DP-Adam only allows for a maximum batch size of about 50. Consequently, given a fixed memory budget, \ouralg{} achieves a $25\%$ increase in throughput compared with DP-Adam, as shown in \cref{tab:vit-timing} in \cref{app:experimental-details-vit}.

\begin{figure*}
    \centering
    \includegraphics[width=\linewidth]{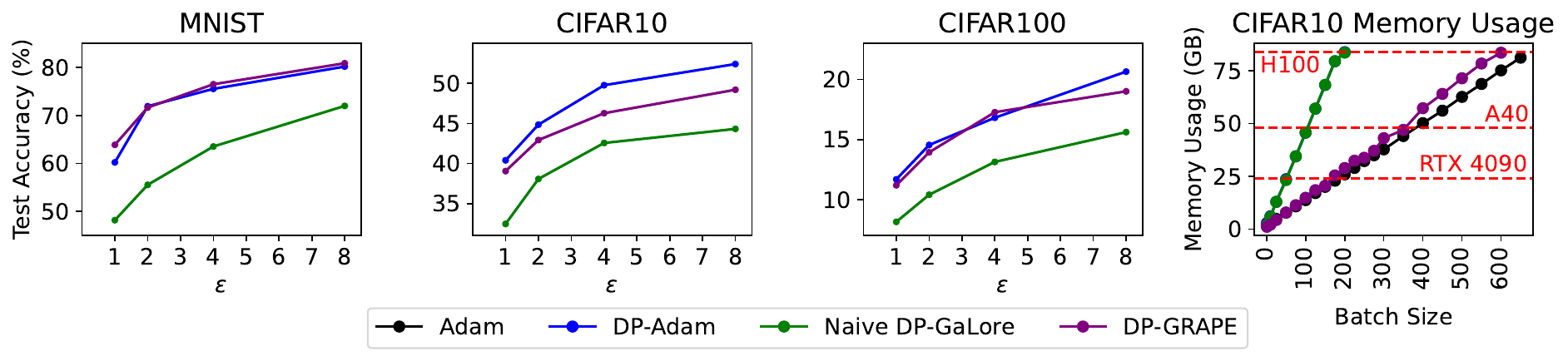}
    \caption{\small Vision Transformer pre-training results for MNIST, CIFAR-10, and CIFAR100 at different $\varepsilon$ privacy levels, and memory usage for different methods during training with varying batch size, with non-private Adam for comparison. Note that the memory usage for DP-Adam is essentially the same as for na\"ive DP-GaLore. For reference, training a ViT without DP achieves around 99\% on MNIST, 93.58\% on CIFAR-10, and 73.81\% on CIFAR-100 \cite{lee2021visiontransformersmallsizedatasets}. See \cref{app:experimental-details-vit} for detailed results in tabular form and the experiment setup.}
    \label{fig:vit_results}
\end{figure*}

\subsection{RoBERTa Fine-Tuning} \label{sec:exp-roberta-finetuning}
We evaluate \ouralg{} on NLP fine-tuning tasks by fine-tuning RoBERTa-Large models (355M parameters) \citep{liu2019roberta} from Hugging Face\footnotemark{} on different sentence classification tasks. Our experimental setup is the same as in \citet{malladi2023fine} and \citet{zhang2024dpzero}: we use a few-shot setting with $512$ samples for each class in all of the datasets. We fine-tune models with \ouralg{} at both $(\eps=2, \delta=1e-5)$ and $(\eps=6,\delta=1e-5)$ privacy. We detail our hyperparameter selection in \cref{app:experimental-details}. \footnotetext{{Link to the checkpoint: \url{https://huggingface.co/FacebookAI/roberta-large}}}

\textbf{Utility:} \Cref{tab:roberta-fewshot-results} shows the results for \ouralg{} along with other DP and non-private baselines. \ouralg{} achieves a higher average test accuracy at both $\varepsilon=2$ and $\varepsilon=6$ than DPZero on all but one of the six datasets we tested on. On average, across the six datasets and two privacy levels, \ouralg{} improves upon the test accuracy of DPZero by $3.7\%$. Furthermore, \ouralg{} is competitive with DP-Adam, achieving an average of a $1.0\%$ increase in accuracy over the different datasets and privacy levels. To investigate the effect of the subspace dimension $r$ on the utility, in \cref{app:experimental-details-roberta} we conduct an ablation study of RoBERTa-Large on SST-2, SNLI, and TREC for $r$ ranging from $4$ to $256$. We observe that the accuracy drops significantly for larger $r$. While for non-private parameter efficient methods such as LoRA \cite{hu2021lora} there is often a small drop in accuracy on fine-tuning tasks for larger dimensions, we hypothesize that in the DP setting the interaction of the variance induced by the projection with the increasing total power of the injected noise for larger $r$ causes more significant degradation. Thus, we recommend that a small $r$ (e.g., $r \leq 16$) is used for fine-tuning with \ouralg{}.

\textbf{Memory Usage:} In addition to achieving comparable utility, \ouralg{} uses significantly less memory than DP-Adam, which we illustrate in \cref{fig:roberta-opt-memory-usage}. When fine-tuning on SST-2 with a batch size of 40, DP-Adam uses 78.1 GB of memory, whereas \ouralg{} uses only 24.4 GB of memory, a $68.7\%$ reduction. While DPZero is very memory efficient, it takes about 10 times as many iterations to converge, which we illustrate in \cref{fig:roberta-convergence} in \cref{app:experimental-details-roberta}. Consequently, for the same experimental setup we use to generate the test results, even though DPZero has cheaper iterations, it takes almost 3 times as long as \ouralg{} to run (\cref{tab:roberta-timing}). As compared to the Vision Transformer pre-training, there is a bigger gap in memory usage between \ouralg{} and non-private Adam, which is due to gradients from the embedding layers and language modeling head not being projected. In \cref{app:experimental-details-roberta} we also show how varying the subspace dimension $r$ affects the total memory usage of \ouralg{}.

\begin{table}[t]
\caption{\small Mean and standard error of final test accuracy over three different seeds for few-shot ($k=512$) fine-tuning of RoBERTa-Large on different datasets, for different DP and non-private methods. The best DP result for each privacy level and dataset is in bold. See \cref{app:experimental-details-roberta} for experiment details.}
\begin{center}
\begin{small}
\setlength{\tabcolsep}{4pt}
\begin{tabular}{lcccccc}
\toprule
\textbf{Task} & \textbf{SST-2} & \textbf{SST-5} & \textbf{SNLI} & \textbf{MNLI} & \textbf{RTE} & \textbf{TREC} \\
\midrule
AdamW (non-private) & $93.1 \pm 0.3$ & $56.6 \pm 0.3$ & $86.4 \pm 0.8$ & $81.4 \pm 0.9$ & $83.6 \pm 1.6$ & $95.9 \pm 0.2$ \\
DP-Adam ($\varepsilon=6$) & $91.6 \pm 1.2$ & $49.0 \pm 0.3$ & $81.5 \pm 1.4$ & $76.3 \pm 0.9$ & $\mathbf{77.3 \pm 1.1}$ & $89.9 \pm 0.8$ \\
DP-Adam ($\varepsilon=2$) & $90.5 \pm 1.5$ & $\mathbf{47.5 \pm 0.5}$ & $74.6 \pm 1.0$ & $\mathbf{70.3 \pm 0.8}$ & $\mathbf{72.8 \pm 0.9}$ & $85.0 \pm 0.5$ \\
\midrule
LoRA (non-private) & $93.3 \pm 0.4$ & $55.3 \pm 1.0$ & $85.9 \pm 0.7$ & $82.2 \pm 0.7$ & $84.2 \pm 0.4$ & $94.6 \pm 0.4$ \\
DP-LoRA ($\varepsilon=6$) & $91.0 \pm 1.3$ & $48.8 \pm 0.5$ & $81.0 \pm 1.5$ & $72.8 \pm 1.8$ & $74.7 \pm 1.3$ & $89.2 \pm 0.8$ \\
DP-LoRA ($\varepsilon=2$) & $90.2 \pm 1.2$ & $47.1 \pm 0.4$ & $74.7 \pm 1.6$ & $65.7 \pm 0.9$ & $69.2 \pm 1.1$ & $83.2 \pm 2.3$ \\
\midrule
MeZO (non-private) & $92.5 \pm 0.3$ & $50.8 \pm 0.8$ & $80.4 \pm 0.6$ & $69.2 \pm 0.3$ & $72.8 \pm 1.0$ & $88.9 \pm 0.1$ \\
DPZero ($\varepsilon=6$) & $92.2 \pm 0.3$ & $\mathbf{49.3 \pm 0.6}$ & $77.8 \pm 1.0$ & $67.4 \pm 0.3$ & $71.9 \pm 0.9$ & $87.6 \pm 0.9$ \\
DPZero ($\varepsilon=2$) & $91.8 \pm 0.1$ & $47.1 \pm 0.9$ & $73.6 \pm 0.9$ & $62.7 \pm 0.9$ & $70.4 \pm 0.7$ & $82.0 \pm 1.6$ \\
\midrule
\ouralg{} ($\varepsilon=6$) & $\mathbf{93.3 \pm 0.4}$ & $49.1 \pm 0.1$ & $\mathbf{83.5 \pm 0.4}$ & $\mathbf{76.7 \pm 0.4}$ & $76.4 \pm 0.8$ & $\mathbf{92.7 \pm 1.0}$ \\
\ouralg{} ($\varepsilon=2$) & $\mathbf{92.6 \pm 0.5}$ & $44.5 \pm 0.4$ & $\mathbf{79.6 \pm 0.4}$ & $68.8 \pm 1.2$ & $\mathbf{72.8 \pm 0.9}$ & $\mathbf{88.1 \pm 2.2}$ \\
\midrule
Zero-Shot & $79.0$ & $35.5$ & $50.2$ & $48.8$ & $51.4$ & $32.0$ \\
\bottomrule
\end{tabular}
\end{small}
\label{tab:roberta-fewshot-results}
\end{center}
\end{table}

\subsection{OPT Fine-Tuning} \label{sec:exp-opt-finetuning}
To assess the scalability of \ouralg{} to larger models, we use it to fine-tune OPT models \citep{zhang2022opt} with 1.3B, 2.7B, and 6.7B parameters from Hugging Face\footnote{Link to the checkpoints: \url{https://huggingface.co/collections/facebook/opt-66ed00e15599f02966818844}} on both classification and generation tasks. We use the same setup as in \citet{malladi2023fine} and \citet{zhang2024dpzero}. For experimental details, see \cref{app:experimental-details-opt}. \Cref{tab:opt-classification-results} and \cref{tab:opt-generation-results} show the results for \ouralg{} and baselines for the classification and generation tasks, respectively.

\begin{figure}[tb!]
    \centering
    \includegraphics[width=0.8\linewidth]{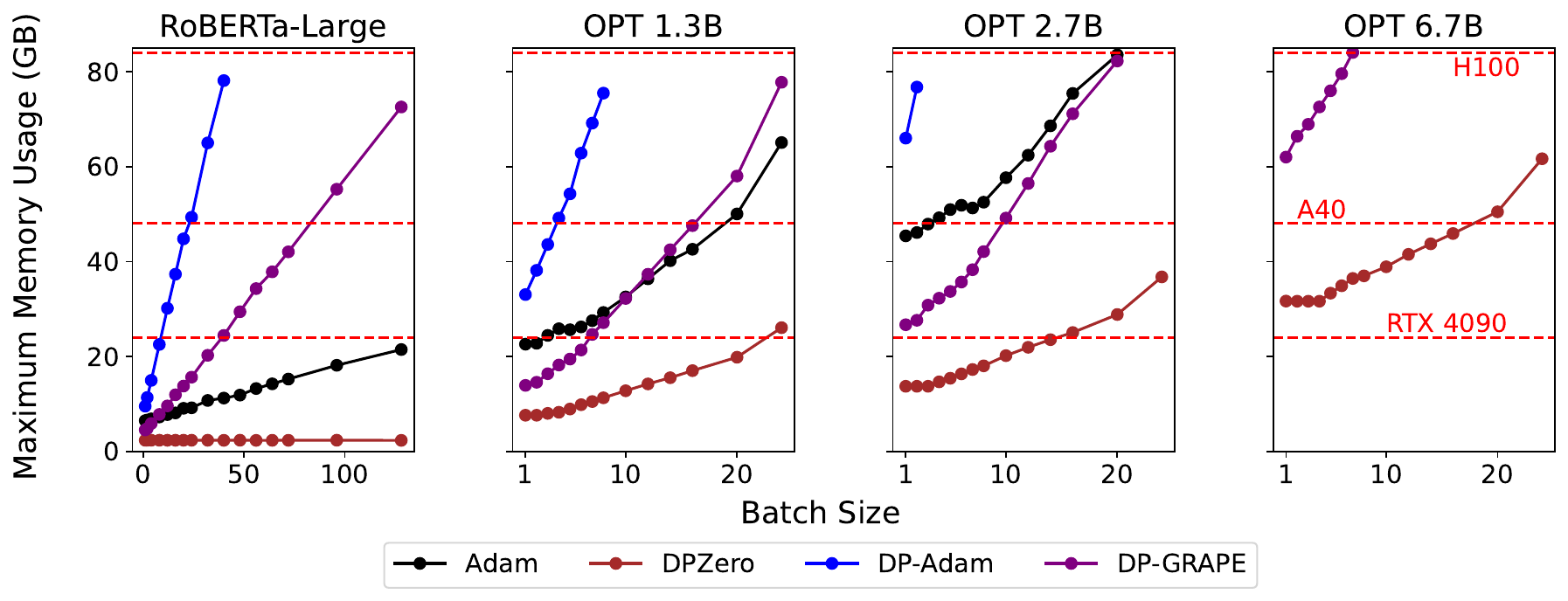}
    \caption{\small Maximum memory usage for fine-tuning RoBERTa-Large on SST-2 and OPT models on SQuAD using Adam, DP-Adam, DPZero, and \ouralg{} with varying batch size. See \cref{app:experimental-details-roberta} and \cref{app:experimental-details-opt} for details.}
    \label{fig:roberta-opt-memory-usage}
\end{figure}

\textbf{Scaling to Larger Models:} While fine-tuning all parameters of the 6.7B model on any of the datasets with DP-Adam or even non-private Adam on a single 80GB GPU exceeds the available memory, \ouralg{} scales to the 6.7B model. We show the memory usage for \ouralg{} and comparison methods with different model sizes in \cref{fig:roberta-opt-memory-usage}. \ouralg{} achieves better utility than DPZero on 18 of the 24 total combinations of model sizes, datasets, and privacy levels, and better utility than DP-Adam on 9 of 16 total combinations (excluding the 6.7B model, for which DP-Adam cannot be used). Furthermore, because \ouralg{} requires 10 times fewer iterations to converge as DP-Zero, it reduces the total fine-tuning time for the 6.7B model on SQuAD using a single H100 GPU by more than 6 times, as shown in \cref{tab:opt-timing} in \cref{app:experimental-details-opt}.

\begin{table}[tb!]
\caption{\small Mean and standard error of final test accuracy over three different seeds for few-shot ($k=1000$) fine-tuning of OPT models on SST-2 and BoolQ classification tasks, for different DP and non-private methods. The best DP result for each privacy level and dataset is in bold. OOM indicates out-of-memory on an 80GB GPU with a batch size of 1 and gradient accumulation. See \cref{app:experimental-details-opt} for experiment details.}
\begin{center}
\begin{small}
\setlength{\tabcolsep}{4pt}
\begin{tabular}{lcccccc}
\toprule
\textbf{Model} & \multicolumn{2}{c}{OPT-1.3B} & \multicolumn{2}{c}{OPT-2.7B} & \multicolumn{2}{c}{OPT-6.7B} \\
\textbf{Task} & \textbf{SST-2} & \textbf{BoolQ} & \textbf{SST-2} & \textbf{BoolQ} & \textbf{SST-2} & \textbf{BoolQ} \\
\midrule
MeZO (non-private) & $88.2 \pm 0.9$ & $63.2 \pm 0.8$ & $91.9 \pm 0.5$ & $65.3 \pm 1.3$ & $93.0 \pm 0.2$ & $67.4 \pm 2.3$ \\
\midrule
DPZero ($\varepsilon=6$) & $88.2 \pm 1.1$ & $62.4 \pm 0.8$ & $91.5 \pm 1.7$ & $\mathbf{65.4 \pm 1.6}$ & $92.6 \pm 0.7$ & $\mathbf{66.8 \pm 1.6}$ \\
DPZero ($\varepsilon=2$) & $86.8 \pm 1.7$ & $61.6 \pm 1.1$ & $90.5 \pm 0.9$ & $\mathbf{63.7 \pm 0.7}$ & $90.6 \pm 1.3$ & $\mathbf{63.7 \pm 0.7}$ \\
\midrule
DP-Adam ($\varepsilon=6$) & $\mathbf{91.2 \pm 0.2}$ & $62.3 \pm 0.4$ & $\mathbf{93.2 \pm 0.1}$ & $62.9 \pm 0.4$ & OOM & OOM \\
DP-Adam ($\varepsilon=2$) & $85.9 \pm 0.8$ & $59.8 \pm 0.7$ & $92.4 \pm 0.04$ & $62.5 \pm 0.3$ & OOM & OOM \\
\midrule
\ouralg{} ($\varepsilon=6$) & $90.8 \pm 0.5$ & $\mathbf{62.5 \pm 0.3}$ & $93.0 \pm 0.5$ & $62.7 \pm 0.6$ & $\mathbf{94.2 \pm 0.3}$ & $63.4 \pm 0.8$ \\
\ouralg{} ($\varepsilon=2$) & $\mathbf{90.5 \pm 0.3}$ & $\mathbf{62.0 \pm 0.4}$ & $\mathbf{92.5 \pm 0.3}$ & $61.6 \pm 1.3$ & $\mathbf{91.2 \pm 0.2}$ & $\mathbf{63.7 \pm 0.2}$ \\
\midrule
Zero-Shot & $53.6$ & $45.3$ & $56.3$ & $47.7$ & $61.2$ & $59.4$ \\
\bottomrule
\end{tabular}
\end{small}
\label{tab:opt-classification-results}
\end{center}
\end{table}

\vskip 0.1in
\begin{table*}[!ht]
\caption{\small Mean and standard error of final f1 score over three different seeds for few-shot ($k=1000$) fine-tuning of OPT models on SQuAD and DROP generation tasks, for different DP and non-private methods. The best DP result for each privacy level and dataset is in bold. OOM indicates out of memory on an 80GB GPU with a batch size of 1 and gradient accumulation.}

\begin{center}
\begin{small}
\setlength{\tabcolsep}{4pt}
\begin{tabular}{lcccccc}
\toprule
\textbf{Model} & \multicolumn{2}{c}{OPT-1.3B} & \multicolumn{2}{c}{OPT-2.7B} & \multicolumn{2}{c}{OPT-6.7B} \\
\textbf{Task} & \textbf{SQuAD} & \textbf{DROP} & \textbf{SQuAD} & \textbf{DROP} & \textbf{SQuAD} & \textbf{DROP} \\
\midrule
MeZO (non-private) & $73.5 \pm 1.2$ & $24.4 \pm 0.2$ & $76.3 \pm 0.8$ & $25.5 \pm 1.2$ & $79.7 \pm 1.1$ & $28.8 \pm 0.7$ \\
\midrule
DPZero ($\varepsilon=6$) & $72.6 \pm 0.8$ & $24.7 \pm 1.0$ & $75.7 \pm 1.5$ & $24.6 \pm 0.5$ & $\mathbf{79.5 \pm 0.9}$ & $\mathbf{28.4 \pm 1.3}$ \\
DPZero ($\varepsilon=2$) & $70.1 \pm 1.6$ & $23.9 \pm 1.2$ & $71.9 \pm 1.2$ & $23.1 \pm 0.9$ & $77.1 \pm 1.0$ & $27.6 \pm 0.7$ \\
\midrule
DP-Adam ($\varepsilon=6$) & $76.9 \pm 0.2$ & $25.9 \pm 1.2$ & $81.4 \pm 0.7$ & $\mathbf{26.3 \pm 1.1}$ & OOM & OOM \\
DP-Adam ($\varepsilon=2$) & $74.1 \pm 0.2$ & $\mathbf{25.2 \pm 1.8}$ & $77.8 \pm 0.4$ & $\mathbf{24.9 \pm 0.7}$ & OOM & OOM \\
\midrule
\ouralg{} ($\varepsilon=6$) & $\mathbf{77.2 \pm 0.1}$ & $\mathbf{26.1 \pm 1.3}$ & $\mathbf{82.0 \pm 0.3}$ & $25.5 \pm 1.2$ & $79.5 \pm 0.2$ & $28.2 \pm 1.3$ \\
\ouralg{} ($\varepsilon=2$) & $\mathbf{76.7 \pm 0.7}$ & $25.0 \pm 1.1$ & $\mathbf{79.2 \pm 0.7}$ & $24.1 \pm 0.3$ & $\mathbf{77.6 \pm 0.4}$ & $\mathbf{27.8 \pm 0.6}$ \\
\midrule
Zero-Shot & $26.8$ & $11.1$ & $29.8$ & $9.7$ & $36.5$ & $17.8$ \\
\bottomrule
\end{tabular}
\end{small}
\label{tab:opt-generation-results}
\end{center}
\end{table*}
\vskip 0.05in

It is important to note that there are two key differences between \ouralg{} and non-private Adam that affect the relative total memory usage: The formation of the sample gradients, and the optimizer states. Because we must form the sample gradients for \ouralg{} to enforce DP, the gradient memory increases linearly with batch size (although with a much smaller slope than for DP-Adam due to projection). Because only the batch gradient is required for non-private Adam, the gradient memory is constant with respect to the batch size, and increases in memory for increasing batch sizes are dominated by the activations. However, \ouralg{} reduces memory as compared to non-private Adam with the use of projected optimizer states. Because the memory usage of optimizer states depends on the size of the model, we can conclude that the memory usage of DP-GRAPE will be the most favorable as compared to non-private Adam for large models with small batch sizes.

We can apply this principle to the recorded memory usage for RoBERTa-Large and the OPT models shown in \cref{fig:roberta-opt-memory-usage}. For RoBERTa-Large, we can observe that DP-GRAPE uses slightly less memory than non-private Adam for the smallest batch sizes, but as the batch size becomes large, the memory usage of the (projected) sample gradients dominate the memory savings from using projected optimizer states, and hence DP-GRAPE uses more memory. Conversely, for the larger OPT models, using projected optimizer states confers a more significant memory reduction, which explains why we observe a large memory reduction versus Adam for small batch sizes, but similar or slightly increased memory usage for the larger batch sizes.

\section{Conclusion} \label{sec:conclusion}
We have introduced \ouralg{}, a memory-efficient DP training method that achieves utility comparable to standard first-order DP methods while significantly reducing the memory usage of per-sample gradients and optimizer states. We experimentally verify \ouralg{} on a variety of tasks including pre-training Vision Transformers, fine-tuning RoBERTa-Large on text classification tasks, and fine-tuning OPT models of different sizes on text classification and generation tasks. For the OPT fine-tuning, \ouralg{} is able to scale to a model size of 6.7B parameters, on which DP-Adam exceeds available memory even with a batch size of 1. Theoretically, \ouralg{} achieves a similar privacy-utility trade-off to DP-SGD. 

In the future, orthogonal techniques like Ghost Clipping \citep{li2021large} and Book-Keeping \citep{bu2023differentially}, which reduce memory by avoiding per-sample gradient instantiation (but do not change the underlying optimization algorithm as \ouralg{} does), can be combined with \ouralg{} to further lower memory usage at the cost of increased per-iteration time. Since they only affect gradient computation and not the optimization algorithm, the utility-privacy trade-off of \ouralg{} remains unchanged. Another possible extension could be applying the approach of \cite{bu2021fast} that estimates the gradients and their norms using random projections. In our setting, we could apply the method to the projected gradients and apply their proposed randomized clipping procedure. Conditioning on the projection matrix, we can directly reuse their privacy analysis to obtain the corresponding privacy guarantees for our method. Because our gradients live in an even lower-dimensional projected space, this approach would provide additional memory savings during norm estimation and clipping.

Other potential improvements to \ouralg{} include adapting the projection dimension across layers based on a selection criterion and utilizing non-Gaussian projection matrices, such as those sampled from a Stiefel manifold. Overall, by reducing resource requirements, \ouralg{} empowers resource-constrained communities and institutions to build and leverage large models while ensuring data privacy, democratizing access to privacy-preserving AI.

\subsubsection*{Broader Impact Statement}
This work introduces a memory-efficient differentially private algorithm for training large models, addressing a significant barrier to equitable AI development. By reducing resource requirements, it empowers resource-constrained communities and institutions to build and leverage large models while ensuring data privacy. This advancement democratizes access to privacy-preserving AI, fostering inclusivity and enabling broader participation in the development of impactful machine learning solutions.

\section{Acknowledgments}
This work was partially supported by a gift from Meta and a gift from Amazon.

\bibliography{ref}

\begin{thebibliography}{41}
\providecommand{\natexlab}[1]{#1}
\providecommand{\url}[1]{\texttt{#1}}
\expandafter\ifx\csname urlstyle\endcsname\relax
  \providecommand{\doi}[1]{doi: #1}\else
  \providecommand{\doi}{doi: \begingroup \urlstyle{rm}\Url}\fi

\bibitem[Abadi et~al.(2016)Abadi, Chu, Goodfellow, McMahan, Mironov, Talwar,
  and Zhang]{abadi2016deep}
Martin Abadi, Andy Chu, Ian Goodfellow, H~Brendan McMahan, Ilya Mironov, Kunal
  Talwar, and Li~Zhang.
\newblock Deep learning with differential privacy.
\newblock In \emph{Proceedings of the 2016 ACM SIGSAC conference on computer
  and communications security}, pp.\  308--318, 2016.

\bibitem[Bao et~al.(2023)Bao, Pittaluga, BG, and Bindschaedler]{bao2023dp}
Wenxuan Bao, Francesco Pittaluga, Vijay~Kumar BG, and Vincent Bindschaedler.
\newblock Dp-mix: mixup-based data augmentation for differentially private
  learning.
\newblock \emph{Advances in Neural Information Processing Systems},
  36:\penalty0 12154--12170, 2023.

\bibitem[Bassily et~al.(2019)Bassily, Feldman, Talwar, and
  Guha~Thakurta]{bassily2019private}
Raef Bassily, Vitaly Feldman, Kunal Talwar, and Abhradeep Guha~Thakurta.
\newblock Private stochastic convex optimization with optimal rates.
\newblock In H.~Wallach, H.~Larochelle, A.~Beygelzimer, F.~d\textquotesingle
  Alch\'{e}-Buc, E.~Fox, and R.~Garnett (eds.), \emph{Advances in Neural
  Information Processing Systems}, volume~32. Curran Associates, Inc., 2019.
\newblock URL
  \url{https://proceedings.neurips.cc/paper_files/paper/2019/file/3bd8fdb090f1f5eb66a00c84dbc5ad51-Paper.pdf}.

\bibitem[Bu et~al.(2021)Bu, Gopi, Kulkarni, Lee, Shen, and
  Tantipongpipat]{bu2021fast}
Zhiqi Bu, Sivakanth Gopi, Janardhan Kulkarni, Yin~Tat Lee, Hanwen Shen, and
  Uthaipon Tantipongpipat.
\newblock Fast and memory efficient differentially private-sgd via jl
  projections.
\newblock \emph{Advances in Neural Information Processing Systems},
  34:\penalty0 19680--19691, 2021.

\bibitem[Bu et~al.(2023)Bu, Wang, Zha, and Karypis]{bu2023differentially}
Zhiqi Bu, Yu-Xiang Wang, Sheng Zha, and George Karypis.
\newblock Differentially private optimization on large model at small cost.
\newblock In \emph{International Conference on Machine Learning}, pp.\
  3192--3218. PMLR, 2023.

\bibitem[Bu et~al.(2024)Bu, Zhang, Hong, Zha, and Karypis]{bu2024pre}
Zhiqi Bu, Xinwei Zhang, Mingyi Hong, Sheng Zha, and George Karypis.
\newblock Pre-training differentially private models with limited public data.
\newblock \emph{arXiv preprint arXiv:2402.18752}, 2024.

\bibitem[Bun et~al.(2018)Bun, Dwork, Rothblum, and Steinke]{bun2018composable}
Mark Bun, Cynthia Dwork, Guy~N Rothblum, and Thomas Steinke.
\newblock Composable and versatile privacy via truncated cdp.
\newblock In \emph{Proceedings of the 50th Annual ACM SIGACT Symposium on
  Theory of Computing}, pp.\  74--86, 2018.

\bibitem[Carlini et~al.(2021)Carlini, Tramer, Wallace, Jagielski, Herbert-Voss,
  Lee, Roberts, Brown, Song, Erlingsson, et~al.]{carlini2021extracting}
Nicholas Carlini, Florian Tramer, Eric Wallace, Matthew Jagielski, Ariel
  Herbert-Voss, Katherine Lee, Adam Roberts, Tom Brown, Dawn Song, Ulfar
  Erlingsson, et~al.
\newblock Extracting training data from large language models.
\newblock In \emph{30th USENIX Security Symposium (USENIX Security 21)}, pp.\
  2633--2650, 2021.

\bibitem[De et~al.(2022)De, Berrada, Hayes, Smith, and Balle]{de2022unlocking}
Soham De, Leonard Berrada, Jamie Hayes, Samuel~L Smith, and Borja Balle.
\newblock Unlocking high-accuracy differentially private image classification
  through scale.
\newblock \emph{arXiv preprint arXiv:2204.13650}, 2022.

\bibitem[Deng(2012)]{deng2012mnist}
Li~Deng.
\newblock The mnist database of handwritten digit images for machine learning
  research [best of the web].
\newblock \emph{IEEE signal processing magazine}, 29\penalty0 (6):\penalty0
  141--142, 2012.

\bibitem[Dong et~al.(2022)Dong, Roth, and Su]{dong2022gaussian}
Jinshuo Dong, Aaron Roth, and Weijie~J Su.
\newblock Gaussian differential privacy.
\newblock \emph{Journal of the Royal Statistical Society Series B: Statistical
  Methodology}, 84\penalty0 (1):\penalty0 3--37, 2022.

\bibitem[Dwork et~al.(2006{\natexlab{a}})Dwork, Kenthapadi, McSherry, Mironov,
  and Naor]{dwork2006our}
Cynthia Dwork, Krishnaram Kenthapadi, Frank McSherry, Ilya Mironov, and Moni
  Naor.
\newblock Our data, ourselves: Privacy via distributed noise generation.
\newblock In \emph{Advances in Cryptology-EUROCRYPT 2006: 24th Annual
  International Conference on the Theory and Applications of Cryptographic
  Techniques, St. Petersburg, Russia, May 28-June 1, 2006. Proceedings 25},
  pp.\  486--503. Springer, 2006{\natexlab{a}}.

\bibitem[Dwork et~al.(2006{\natexlab{b}})Dwork, McSherry, Nissim, and
  Smith]{dwork2006calibrating}
Cynthia Dwork, Frank McSherry, Kobbi Nissim, and Adam Smith.
\newblock Calibrating noise to sensitivity in private data analysis.
\newblock In \emph{Theory of Cryptography: Third Theory of Cryptography
  Conference, TCC 2006, New York, NY, USA, March 4-7, 2006. Proceedings 3},
  pp.\  265--284. Springer, 2006{\natexlab{b}}.

\bibitem[Dwork et~al.(2014)Dwork, Roth, et~al.]{dwork2014algorithmic}
Cynthia Dwork, Aaron Roth, et~al.
\newblock The algorithmic foundations of differential privacy.
\newblock \emph{Foundations and Trends{\textregistered} in Theoretical Computer
  Science}, 9\penalty0 (3--4):\penalty0 211--407, 2014.

\bibitem[Fredrikson et~al.(2015)Fredrikson, Jha, and
  Ristenpart]{fredrikson2015model}
Matt Fredrikson, Somesh Jha, and Thomas Ristenpart.
\newblock Model inversion attacks that exploit confidence information and basic
  countermeasures.
\newblock In \emph{Proceedings of the 22nd ACM SIGSAC conference on computer
  and communications security}, pp.\  1322--1333, 2015.

\bibitem[Greene(2003)]{greene2003econometric}
William~H Greene.
\newblock Econometric analysis.
\newblock \emph{Pretence Hall}, 2003.

\bibitem[Gu et~al.(2023)Gu, Kamath, and Wu]{gu2023choosing}
Xin Gu, Gautam Kamath, and Zhiwei~Steven Wu.
\newblock Choosing public datasets for private machine learning via gradient
  subspace distance.
\newblock \emph{arXiv preprint arXiv:2303.01256}, 2023.

\bibitem[Hao et~al.(2024)Hao, Cao, and Mou]{hao2024flora}
Yongchang Hao, Yanshuai Cao, and Lili Mou.
\newblock Flora: Low-rank adapters are secretly gradient compressors.
\newblock \emph{arXiv preprint arXiv:2402.03293}, 2024.

\bibitem[He et~al.(2024)He, Li, Hu, Chen, and Yuan]{he2024subspace}
Yutong He, Pengrui Li, Yipeng Hu, Chuyan Chen, and Kun Yuan.
\newblock Subspace optimization for large language models with convergence
  guarantees.
\newblock \emph{arXiv preprint arXiv:2410.11289}, 2024.

\bibitem[Hu et~al.(2021)Hu, Shen, Wallis, Allen-Zhu, Li, Wang, Wang, and
  Chen]{hu2021lora}
Edward~J Hu, Yelong Shen, Phillip Wallis, Zeyuan Allen-Zhu, Yuanzhi Li, Shean
  Wang, Lu~Wang, and Weizhu Chen.
\newblock Lora: Low-rank adaptation of large language models.
\newblock \emph{arXiv preprint arXiv:2106.09685}, 2021.

\bibitem[Jiang et~al.(2024)Jiang, Hasan, Saadati, Balu, Liu, and
  Sarkar]{jiangd2p2}
Zhanhong Jiang, Md~Zahid Hasan, Nastaran Saadati, Aditya Balu, Chao Liu, and
  Soumik Sarkar.
\newblock D2p2-sgd: Dynamically differentially private projected stochastic
  gradient descent.
\newblock 2024.

\bibitem[Krizhevsky et~al.(2009)Krizhevsky, Hinton,
  et~al.]{krizhevsky2009learning}
Alex Krizhevsky, Geoffrey Hinton, et~al.
\newblock Learning multiple layers of features from tiny images.
\newblock 2009.

\bibitem[Lee \& Kifer(2021)Lee and Kifer]{lee2021scaling}
Jaewoo Lee and Daniel Kifer.
\newblock Scaling up differentially private deep learning with fast per-example
  gradient clipping.
\newblock \emph{Proceedings on Privacy Enhancing Technologies}, 2021.

\bibitem[Lee et~al.(2021)Lee, Lee, and
  Song]{lee2021visiontransformersmallsizedatasets}
Seung~Hoon Lee, Seunghyun Lee, and Byung~Cheol Song.
\newblock Vision transformer for small-size datasets, 2021.
\newblock URL \url{https://arxiv.org/abs/2112.13492}.

\bibitem[Li et~al.(2021)Li, Tramer, Liang, and Hashimoto]{li2021large}
Xuechen Li, Florian Tramer, Percy Liang, and Tatsunori Hashimoto.
\newblock Large language models can be strong differentially private learners.
\newblock \emph{arXiv preprint arXiv:2110.05679}, 2021.

\bibitem[Li et~al.(2025)Li, Zhang, Zhong, Deng, Razaviyayn, and
  Mirrokni]{li2024addax}
Zeman Li, Xinwei Zhang, Peilin Zhong, Yuan Deng, Meisam Razaviyayn, and Vahab
  Mirrokni.
\newblock Addax: Utilizing zeroth-order gradients to improve memory efficiency
  and performance of sgd for fine-tuning language models.
\newblock \emph{International Conference on Learning Representations (ICLR)},
  2025.

\bibitem[Lin et~al.(2020)Lin, Jin, and Jordan]{lin2020gradient}
Tianyi Lin, Chi Jin, and Michael Jordan.
\newblock On gradient descent ascent for nonconvex-concave minimax problems.
\newblock In \emph{International Conference on Machine Learning}, pp.\
  6083--6093. PMLR, 2020.

\bibitem[Liu(2019)]{liu2019roberta}
Yinhan Liu.
\newblock Roberta: A robustly optimized bert pretraining approach.
\newblock \emph{arXiv preprint arXiv:1907.11692}, 364, 2019.

\bibitem[Lowy et~al.(2024)Lowy, Ullman, and Wright]{luw24privnonconvex}
Andrew Lowy, Jonathan Ullman, and Stephen~J. Wright.
\newblock How to make the gradients small privately: improved rates for
  differentially private non-convex optimization.
\newblock In \emph{Proceedings of the 41st International Conference on Machine
  Learning}, ICML'24. JMLR.org, 2024.

\bibitem[Malladi et~al.(2023)Malladi, Gao, Nichani, Damian, Lee, Chen, and
  Arora]{malladi2023fine}
Sadhika Malladi, Tianyu Gao, Eshaan Nichani, Alex Damian, Jason~D Lee, Danqi
  Chen, and Sanjeev Arora.
\newblock Fine-tuning language models with just forward passes.
\newblock \emph{Advances in Neural Information Processing Systems},
  36:\penalty0 53038--53075, 2023.

\bibitem[Mireshghallah et~al.(2022)Mireshghallah, Uniyal, Wang, Evans, and
  Berg-Kirkpatrick]{mireshghallah2022memorization}
Fatemehsadat Mireshghallah, Archit Uniyal, Tianhao Wang, David Evans, and
  Taylor Berg-Kirkpatrick.
\newblock Memorization in nlp fine-tuning methods.
\newblock \emph{arXiv preprint arXiv:2205.12506}, 2022.

\bibitem[Pasand \& Bashivan(2024)Pasand and Bashivan]{pasand2024rgp}
Ali~Saheb Pasand and Pouya Bashivan.
\newblock Rgp: Achieving memory-efficient model fine-tuning via randomized
  gradient projection.
\newblock In \emph{NeurIPS Efficient Natural Language and Speech Processing
  Workshop}, pp.\  47--54. PMLR, 2024.

\bibitem[Paszke et~al.(2019)Paszke, Gross, Massa, Lerer, Bradbury, Chanan,
  Killeen, Lin, Gimelshein, Antiga, et~al.]{paszke2019pytorch}
Adam Paszke, Sam Gross, Francisco Massa, Adam Lerer, James Bradbury, Gregory
  Chanan, Trevor Killeen, Zeming Lin, Natalia Gimelshein, Luca Antiga, et~al.
\newblock Pytorch: An imperative style, high-performance deep learning library.
\newblock \emph{Advances in neural information processing systems}, 32, 2019.

\bibitem[Tang et~al.(2024)Tang, Panda, Nasr, Mahloujifar, and
  Mittal]{tang2024private}
Xinyu Tang, Ashwinee Panda, Milad Nasr, Saeed Mahloujifar, and Prateek Mittal.
\newblock Private fine-tuning of large language models with zeroth-order
  optimization.
\newblock \emph{arXiv preprint arXiv:2401.04343}, 2024.

\bibitem[Wang et~al.(2017)Wang, Ye, and Xu]{wang2017differentially}
Di~Wang, Minwei Ye, and Jinhui Xu.
\newblock Differentially private empirical risk minimization revisited: faster
  and more general.
\newblock In \emph{Proceedings of the 31st International Conference on Neural
  Information Processing Systems}, NIPS'17, pp.\  2719–2728, Red Hook, NY,
  USA, 2017. Curran Associates Inc.
\newblock ISBN 9781510860964.

\bibitem[Yu et~al.(2021{\natexlab{a}})Yu, Naik, Backurs, Gopi, Inan, Kamath,
  Kulkarni, Lee, Manoel, Wutschitz, et~al.]{yu2021differentially}
Da~Yu, Saurabh Naik, Arturs Backurs, Sivakanth Gopi, Huseyin~A Inan, Gautam
  Kamath, Janardhan Kulkarni, Yin~Tat Lee, Andre Manoel, Lukas Wutschitz,
  et~al.
\newblock Differentially private fine-tuning of language models.
\newblock \emph{arXiv preprint arXiv:2110.06500}, 2021{\natexlab{a}}.

\bibitem[Yu et~al.(2021{\natexlab{b}})Yu, Zhang, Chen, and Liu]{yu2021not}
Da~Yu, Huishuai Zhang, Wei Chen, and Tie-Yan Liu.
\newblock Do not let privacy overbill utility: Gradient embedding perturbation
  for private learning.
\newblock \emph{arXiv preprint arXiv:2102.12677}, 2021{\natexlab{b}}.

\bibitem[Zhang et~al.(2024)Zhang, Li, Thekumparampil, Oh, and
  He]{zhang2024dpzero}
Liang Zhang, Bingcong Li, Kiran~Koshy Thekumparampil, Sewoong Oh, and Niao He.
\newblock Dpzero: Private fine-tuning of language models without
  backpropagation.
\newblock In \emph{Forty-first International Conference on Machine Learning},
  2024.

\bibitem[Zhang et~al.(2022)Zhang, Roller, Goyal, Artetxe, Chen, Chen, Dewan,
  Diab, Li, Lin, et~al.]{zhang2022opt}
Susan Zhang, Stephen Roller, Naman Goyal, Mikel Artetxe, Moya Chen, Shuohui
  Chen, Christopher Dewan, Mona Diab, Xian Li, Xi~Victoria Lin, et~al.
\newblock Opt: Open pre-trained transformer language models.
\newblock \emph{arXiv preprint arXiv:2205.01068}, 2022.

\bibitem[Zhao et~al.(2024)Zhao, Zhang, Chen, Wang, Anandkumar, and
  Tian]{zhao2024galore}
Jiawei Zhao, Zhenyu Zhang, Beidi Chen, Zhangyang Wang, Anima Anandkumar, and
  Yuandong Tian.
\newblock Galore: Memory-efficient llm training by gradient low-rank
  projection.
\newblock \emph{arXiv preprint arXiv:2403.03507}, 2024.

\bibitem[Zhou et~al.(2020)Zhou, Wu, and Banerjee]{zhou2020bypassing}
Yingxue Zhou, Zhiwei~Steven Wu, and Arindam Banerjee.
\newblock Bypassing the ambient dimension: Private sgd with gradient subspace
  identification.
\newblock \emph{arXiv preprint arXiv:2007.03813}, 2020.

\end{thebibliography}
\bibliographystyle{tmlr}

\appendix

\section{DP-Adam} \label{app:dpadam}
Here, we detail the standard DP-Adam algorithm (with flat clipping) using our notation.
\vskip 0.2cm
\begin{algorithm}
\label{algo:dp-adam}
\caption{DP-Adam}
\begin{algorithmic}[1]
\REQUIRE Dataset $X=\{\xi_1, \dots, \xi_n\}$, model parameters $\{W_{\ell}^0 \}_{\ell=1}^L$, learning rate $\eta$, decay rates $\beta_1, \beta_2$, batch size $B$, total iterations $T$
\FOR {$t=1, 2, \dots, T$}
    \FOR{$\ell = L, L-1, \dots, 1$}
        \STATE $\{ G_{\ell,i}^t \}_{i=1}^B \gets \nabla_{W_{\ell}^t} f(\{ W_{\ell}^t \}_{\ell=1}^L ; \{ \xi_i \}_{i=1}^B)$
    \ENDFOR 
    \STATE $\tilde{G}^t \gets \frac{1}{B}(\sum_{i=1}^B \text{clip}(G_i^t, C) + \mathcal{N}(0, C^2 \sigma^2 I) \in \mathbb{R}^d)$
    \STATE $\alpha^t \gets \eta \frac{\sqrt{1 - \beta_2^t}}{1 - \beta_1^t}$ 
    \FOR{$\ell = 1, 2, \dots, L$} 
        \STATE $M_{\ell}^t \gets \beta_1 M_{\ell}^{t-1} + (1-\beta_1)\tilde{G}_{\ell}^t$ 
        \STATE $V_{\ell}^t \gets \beta_2 V_{\ell}^{t-1} + (1-\beta_2)(\tilde{G}_{\ell}^t)^2$ 
        \STATE $W_{\ell}^t \gets W_{\ell}^{t-1} - \alpha^t \cdot (\frac{M_{\ell}^t}{\sqrt{V_{\ell}^t} + \phi})$
    \ENDFOR
\ENDFOR
\STATE Return $\{ W_{\ell}^T \}_{\ell=1}^L$
\end{algorithmic}
\end{algorithm}

\section{GaLore and Na\"ive DP-GaLore} \label{app:galore-details}
Here, we detail the GaLore algorithm introduced by \citet{zhao2024galore} and the na\"ive version of DP-GaLore that we discuss in \cref{sec:naive-dpgalore}. For na\"ive DP-GaLore, gradients are privatized prior to projection and before the SVD updates, so that the subspaces obtained from the SVD can be used in subsequent iterations with no privacy loss.
\vskip 0.2cm
\begin{algorithm}
\label{algo:galore}
\caption{GaLore}
\begin{algorithmic}[1]
\REQUIRE Dataset $X=\{\xi_1, \dots, \xi_n\}$, model parameters $\{W_{\ell}^0 \}_{\ell=1}^L$, learning rate $\eta$, subspace dimension $r$, subspace change frequency $F$, batch size $B$, total iterations $T$
\FOR {$t=1, 2, \dots, T$}
    \FOR{$\ell = L, L-1, \dots, 1$}
        \STATE $G_{\ell}^t \gets \nabla_{W_{\ell}^t} f(\{ W_{\ell}^t \}_{\ell=1}^L ; \{ \xi_i \}_{i=1}^B)$
        \IF {$t \mod F = 0$}
            \STATE $U, S, V \gets \text{SVD} (G_{\ell}^t)$
            \STATE $P_{\ell}^t \gets U[:, :r]$
        \ELSE
            \STATE $P_{\ell}^t \gets P_{\ell}^{t-1}$ 
        \ENDIF
        \STATE $R_{\ell}^t \gets (P_{\ell}^t)^{\top} G_{\ell}^t $ 
    \ENDFOR 
    \STATE $\{ W_{\ell}^{t+1} \}^L_{\ell = 1} =\mathbf{ProjectedAdamUpdate}(\{ W_{\ell}^t \}^L_{\ell = 1}, \{ R_{\ell}^t \}^L_{\ell = 1}, \{ P_{\ell}^t \}_{\ell=1}^L, \eta)$ 
\ENDFOR
\STATE Return $\{ W_{\ell}^T \}_{\ell=1}^L$
\end{algorithmic}
\end{algorithm}
\clearpage

\begin{algorithm}
\label{algo:naive-dpgalore}
\caption{Na\"ive DP-GaLore}
\begin{algorithmic}[1]
\REQUIRE Dataset $X=\{\xi_1, \dots, \xi_n\}$, model parameters $\{W_{\ell}^0 \}_{\ell=1}^L$, learning rate $\eta$, subspace dimension $r$, subspace change frequency $F$, batch size $B$, clipping parameter $C$, noise level $\sigma$, total iterations $T$
\FOR {$t=1, 2, \dots, T$}
    \FOR{$\ell = L, L-1, \dots, 1$}
        \STATE $\{ G_{\ell,i}^t \}_{i=1}^B \gets \nabla_{W_{\ell}^t} f(\{ W_{\ell}^t \}_{\ell=1}^L ; \{ \xi_i \}_{i=1}^B)$
    \ENDFOR
    \STATE $\tilde{G}^t \gets \frac{1}{B}(\sum_{i=1}^B \text{clip}(G_i^t, C) + \mathcal{N}(0, C^2 \sigma^2 I) \in \mathbb{R}^d)$
    \FOR{$\ell = L, L-1, \dots, 1$}
        \IF {$t \mod F = 0$}
            \STATE $U, S, V \gets \text{SVD} (\tilde{G}_{\ell}^t)$
            \STATE $P_{\ell}^t \gets U[:, :r]$
        \ELSE
            \STATE $P_{\ell}^t \gets P_{\ell}^{t-1}$ 
        \ENDIF
        \STATE $\tilde{R}_{\ell}^t \gets (P_{\ell}^t)^{\top} \tilde{G}_{\ell}^t $ 
    \ENDFOR 
    \STATE $\{ W_{\ell}^{t+1} \}^L_{\ell = 1} =\mathbf{ProjectedAdamUpdate}(\{ W_{\ell}^t \}^L_{\ell = 1}, \{ \tilde{R}_{\ell}^t \}^L_{\ell = 1}, \{ P_{\ell}^t \}_{\ell=1}^L, \eta)$ 
\ENDFOR
\STATE Return $\{ W_{\ell}^T \}_{\ell=1}^L$
\end{algorithmic}
\end{algorithm}

\section{Experiment Details} \label{app:experimental-details}
\subsection{Vision Transformer Training} \label{app:experimental-details-vit}
We use the same grid search for all three methods, selecting the clipping parameter $C$ from $\{0.1, 1, 10\}$ and the learning rate from $\{1\text{e}-4, 5\text{e}-4, 1\text{e}-3, 5\text{e}-3\}$. For the grid-search, we split the training set of each dataset randomly into $80\%$ training and $20\%$ testing data and select the combination of $C$ and learning rate which achieves the highest validation accuracy during training. We use a privacy level of $(\varepsilon=2, \delta=\frac{1}{n})$ for the grid search. The best hyperparameters are then used to train models on the entire training set at the different privacy levels $\varepsilon=1,2,4,8$, which are evaluated on the original testing set. \Cref{tab:vit-grid-search-hyperparameters} lists the selected $C$ and learning rate for each method and dataset. For all experiments (both the grid search and final training runs), we use a total batch size of 1000 (which is achieved through gradient accumulation) and train for 60 epochs. In \cref{tab:vit-full-results} we list the best test accuracy during training for each method on the different datasets and privacy levels. 

\vskip 0.2cm
{
\renewcommand{\arraystretch}{1.25}
\begin{table}[h!]
\caption{\small Clipping parameter $C$ and learning rate selected from grid search for each method and dataset for Vision Transformer pre-training.}
\vskip 0.05in
\centering
\begin{small}
\begin{tabular}{c|ccc|ccc}
\toprule
Method & \multicolumn{3}{c|}{$C$} & \multicolumn{3}{c}{Learning Rate} \\
\cmidrule(lr){2-4} \cmidrule(lr){5-7}
& MNIST & CIFAR-10 & CIFAR-100 & MNIST & CIFAR-10 & CIFAR-100 \\
\midrule
DP-Adam & 10.0 & 0.1 & 1.0 & 1\text{e}-3 & 1\text{e}-3 & 5\text{e}-4 \\
Naïve DP-GaLore & 1.0 & 1.0 & 0.1 & 5\text{e}-4 & 5\text{e}-4 & 1\text{e}-3 \\
\ouralg{} & 0.1 & 1.0 & 10.0 & 5\text{e}-3 & 1\text{e}-3 & 1\text{e}-3 \\
\bottomrule
\end{tabular}
\end{small}
\label{tab:vit-grid-search-hyperparameters}
\end{table}
}

\begin{table}
\caption{\small Vision Transformer pretraining results for MNIST, CIFAR-10, and CIFAR100 at different privacy levels (best test accuracy during training).}
\begin{center}
\begin{small}
\begin{tabular}{lccc}
\toprule
\textbf{Task} & \textbf{MNIST} & \textbf{CIFAR-10} & \textbf{CIFAR-100} \\
\midrule
DP-Adam ($\varepsilon=1$) & 60.2 & 40.4 & 11.7 \\
DP-Adam ($\varepsilon=2$) & 71.9 & 44.9 & 14.5 \\
DP-Adam ($\varepsilon=4$) & 75.5 & 49.8 & 16.8 \\
DP-Adam ($\varepsilon=8$) & 80.2 & 52.4 & 20.7 \\
\midrule
Na\"ive DP-GaLore ($\varepsilon=1$) & 48.2 & 32.5 & 8.1 \\
Na\"ive DP-GaLore ($\varepsilon=2$) & 55.6 & 38.1 & 10.4 \\
Na\"ive DP-GaLore ($\varepsilon=4$) & 63.5 & 42.6 & 13.1 \\
Na\"ive DP-GaLore ($\varepsilon=8$) & 72.0 & 44.3 & 15.6 \\
\midrule
\ouralg{} ($\varepsilon=1$) & 63.9 & 39.1 & 11.2 \\
\ouralg{} ($\varepsilon=2$) & 71.7 & 42.9 & 13.9 \\
\ouralg{} ($\varepsilon=4$) & 76.5 & 46.3 & 17.3 \\
\ouralg{} ($\varepsilon=8$) & 80.9 & 49.2 & 19.0 \\
\bottomrule
\end{tabular}
\end{small}
\label{tab:vit-full-results}
\end{center}
\end{table}

\vskip 0.3cm

For the memory experiment with results shown in \cref{fig:vit_results}, we train for 5 steps and record the maximum memory reserved by PyTorch \citep{paszke2019pytorch} using the \texttt{torch.cuda.max\_memory\_reserved()} function, for a range of batch sizes. For all sizes, we use a gradient accumulation step so that accumulated gradients are included in the memory accounting.

To generate the timing results shown in \cref{tab:vit-timing}, we train each method for 1 epoch and then extrapolate the time taken to complete 60 epochs of training. We match the setup we use to generate the results in \cref{fig:vit_results}, with a total batch size of 1000 that is achieved by gradient accumulation. For Adam and \ouralg{} we use a physical batch size of 500 (which uses 63.0GB and 70.6GB of memory, respectively), and for DP-Adam and Na\"ive DP-GaLore we use a physical batch size of 200 (which uses 83.9 and 83.7GB of memory, respectively).

The memory and timing experiments were conducted on a single H100 GPU.

To create the plot of singular values shown in \cref{fig:vit-svals}, we record the top $64$ (corresponding to the projection dimension we use for all Vision Transformer experiments) singular values for each layer during the first step of training with a batch size of 1000, with possible clipping and different noise levels $\sigma$ applied to the gradients prior to computing the SVD. 

\begin{table}[t]
\caption{\small Number of samples processed per second during training and total train time for Vision Transformer on CIFAR-10 (using 1 H100 GPU).}
\begin{center}
\begin{small}
\begin{tabular}{lcc}
\toprule
\textbf{Method} & \textbf{Throughput (Samples/s)} & \textbf{Total Training Time (hours)} \\
\midrule
Adam (non-private) & 379 & 2.2 \\
\midrule
DP-Adam & 219 & 3.8 \\
Na\"ive DP-GaLore & 217 & 3.8 \\
\ouralg{} & 273 & 3.1 \\
\bottomrule
\end{tabular}
\end{small}
\label{tab:vit-timing}
\end{center}
\end{table}

We also fine-tune a pretrained checkpoint of ViT-Base on CIFAR10 and CIFAR100 \footnote{Link to the checkpoint: \url{https://huggingface.co/google/vit-base-patch16-224}}. For these experiments, we again use a grid search over the training set for all three methods to select the clipping parameter $C$ from $\{0.1, 1, 10\}$ and the learning rate from $\{1\text{e}-5, 5\text{e}-5, 1\text{e}-4, 5\text{e}-4\}$. Using the best hyperparameters for each method, we then fine-tune on the entire training set at the different privacy level $\varepsilon=1,2,4,8$ and evaluate on the original testing set. \Cref{tab:vit-grid-search-hyperparameters-finetuning} lists the selected $C$ and learning rate for each method and dataset. For both the hyperparameter search and the final fine-tuning, we use a total batch size of 1000 and train for 20 epochs. \Cref{tab:vit-full-results-finetuning} lists the final results (best test accuracy) for each method and dataset at different privacy levels.

\vskip 0.2cm
{
\renewcommand{\arraystretch}{1.25}
\begin{table}[h!]
\caption{\small Clipping parameter $C$ and learning rate selected from grid search for each method and dataset for Vision Transformer fine-tuning.}
\vskip 0.05in
\centering
\begin{small}
\begin{tabular}{c|cc|cc}
\toprule
Method & \multicolumn{2}{c|}{$C$} & \multicolumn{2}{c}{Learning Rate} \\
\cmidrule(lr){2-3} \cmidrule(lr){4-5}
& CIFAR-10 & CIFAR-100 & CIFAR-10 & CIFAR-100 \\
\midrule
DP-Adam & 10.0 & 10.0 & 1\text{e}-4 & 5\text{e}-4 \\
Naïve DP-GaLore & 10.0 & 1.0 & 5\text{e}-4 & 5\text{e}-4 \\
\ouralg{} & 0.1 & 1.0 & 5\text{e}-4 & 5\text{e}-4 \\
\bottomrule
\end{tabular}
\end{small}
\label{tab:vit-grid-search-hyperparameters-finetuning}
\end{table}
}

\begin{table}
\caption{\small Vision Transformer fine-tuning results for CIFAR-10, and CIFAR100 at different privacy levels (best test accuracy during training).}
\begin{center}
\begin{small}
\begin{tabular}{lcc}
\toprule
\textbf{Task} & \textbf{CIFAR-10} & \textbf{CIFAR-100} \\
\midrule
DP-Adam ($\varepsilon=1$) & 97.2 & 49.7 \\
DP-Adam ($\varepsilon=2$) & 97.5 & 70.6 \\
DP-Adam ($\varepsilon=4$) & 98.1 & 77.3 \\
DP-Adam ($\varepsilon=8$) & 98.2 & 80.8 \\
\midrule
Na\"ive DP-GaLore ($\varepsilon=1$) & 96.4 & 65.6 \\
Na\"ive DP-GaLore ($\varepsilon=2$) & 97.0 & 78.8 \\
Na\"ive DP-GaLore ($\varepsilon=4$) & 97.5 & 83.8 \\
Na\"ive DP-GaLore ($\varepsilon=8$) & 97.7 & 85.5 \\
\midrule
\ouralg{} ($\varepsilon=1$) & 97.0 & 81.4 \\
\ouralg{} ($\varepsilon=2$) & 97.9 & 85.4 \\
\ouralg{} ($\varepsilon=4$) & 97.8 & 86.9 \\
\ouralg{} ($\varepsilon=8$) & 98.2 & 88.1 \\
\bottomrule
\end{tabular}
\end{small}
\label{tab:vit-full-results-finetuning}
\end{center}
\end{table}

\subsection{RoBERTa Fine-Tuning} \label{app:experimental-details-roberta}
We follow the same experimental setup and build off of the same codebase as used by \citet{zhang2024dpzero} and \citet{malladi2023fine} to fine-tune RoBERTa-Large \citep{liu2019roberta} on datasets that cover sentiment analysis (SST-2, SST-5), natural language inference (SNLI, MNLI, RTE), and topic classification (TREC). For all datasets, we use a few-shot setting with 512 samples per class, and 1000 total test samples. We first complete a grid search to find reasonable values for the projection dimension $r$, the projection update frequency $F$, the learning rate $\eta$, the DP clipping parameter $C$, and the total number of training steps $T$, for the SST-2 and MNLI datasets, evaluating on the development set for seed 100. Based on these experiments, we select $r=16$, $F=100$, $\eta=1\text{e}-4$, and $T=1000$. Using these values, we run a search for only the clipping parameter for the remaining datasets, selecting from $\{ 0.1, 0.5, 1.0, 5.0, 10.0, 20.0\}$ (again evaluating on the development set for seed 100). The grid search and clipping parameter search are done with $(\varepsilon=6,\delta=1\text{e}-5)$ privacy. The best $C$ value from this search is $0.5$ for SST-2, $20.0$ for SST-5, $0.1$ for SNLI, $10.0$ for MNLI, $0.5$ for RTE and $0.5$ for TREC. Using the best $C$ value, we run the final results for each dataset on the seeds 13, 21, and 42 (which contain different samplings of the full datasets), at both $(\varepsilon=2, \delta=1\text{e}-5)$ and $(\varepsilon=6, \delta=1\text{e}-5)$ privacy for each, and record the average final test accuracy over the 3 seeds for each dataset and privacy level. We train with a batch size of $64$ for all experiments, which may be achieved using gradient accumulation. The results for AdamW (non-private), DP-Adam, LoRA (non-private), DP-LoRA, MeZO (non-private), and DPZero come from \citet{zhang2024dpzero}.

For the memory experiment with results shown in in \cref{fig:roberta-opt-memory-usage}, we train for 30 steps and record the maximum memory reserved by PyTorch \citep{paszke2019pytorch} using the \texttt{torch.cuda.max\_memory\_reserved()} function, for a range of batch sizes. For all sizes, we use a gradient accumulation step for the first-order methods so that accumulated gradients are included in the memory accounting. We also repeat the memory experiment for \ouralg{} with varying subspace dimension $r$, with results shown in \cref{fig:roberta-memory-varied-r}. In \cref{fig:roberta-r-ablation}, we plot the test set accuracy of RoBERTa-Large on SST-2, SNLI, and TREC using \ouralg{} with varying $r$ from $[4,8,16,32,64,128,256]$ at $\varepsilon=2$ and $\varepsilon=6$ privacy levels. To obtain the results, we perform an initial search for the optimal clipping parameter $C$ from the choices $[0.1,0.5,1.0,5.0,20.0]$ for each $r$ and dataset, and use that value of $C$ to fine-tune the model on the dataset with three different random seeds, with the average final test accuracy reported.

\vskip 0.15cm
\begin{figure}[ht]
    \centering
    \includegraphics[width=0.6\linewidth]{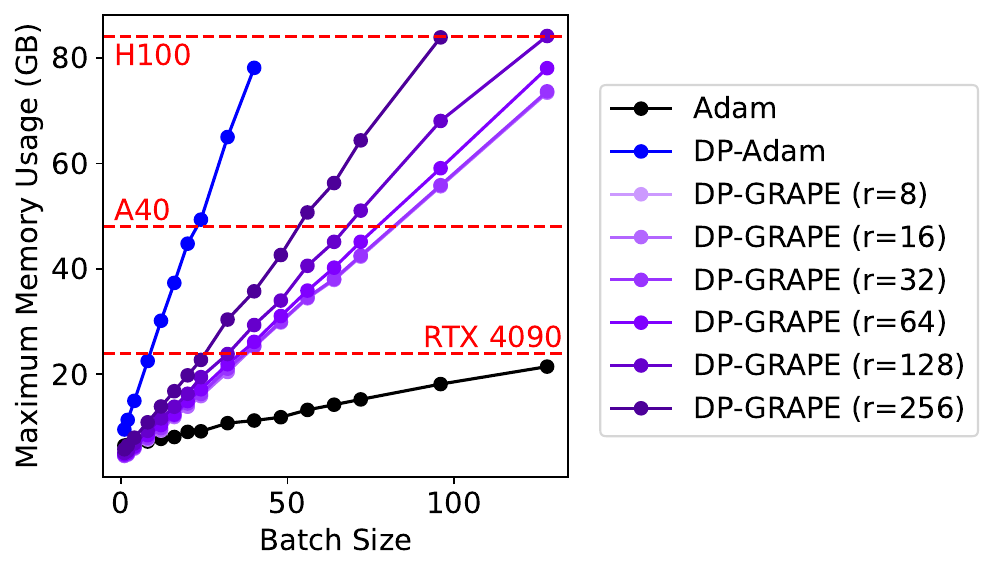}
    \caption{\small Maximum memory usage for fine-tuning RoBERTa-Large with \ouralg{} using different subspace dimensions $r$, with comparisons to Adam, DP-Adam, and DPZero.}
    \label{fig:roberta-memory-varied-r}
\end{figure}
\vskip 0.15cm

\begin{figure}
    \centering
    \begin{minipage}{0.325\linewidth}
        \centering
        \includegraphics[width=\linewidth]{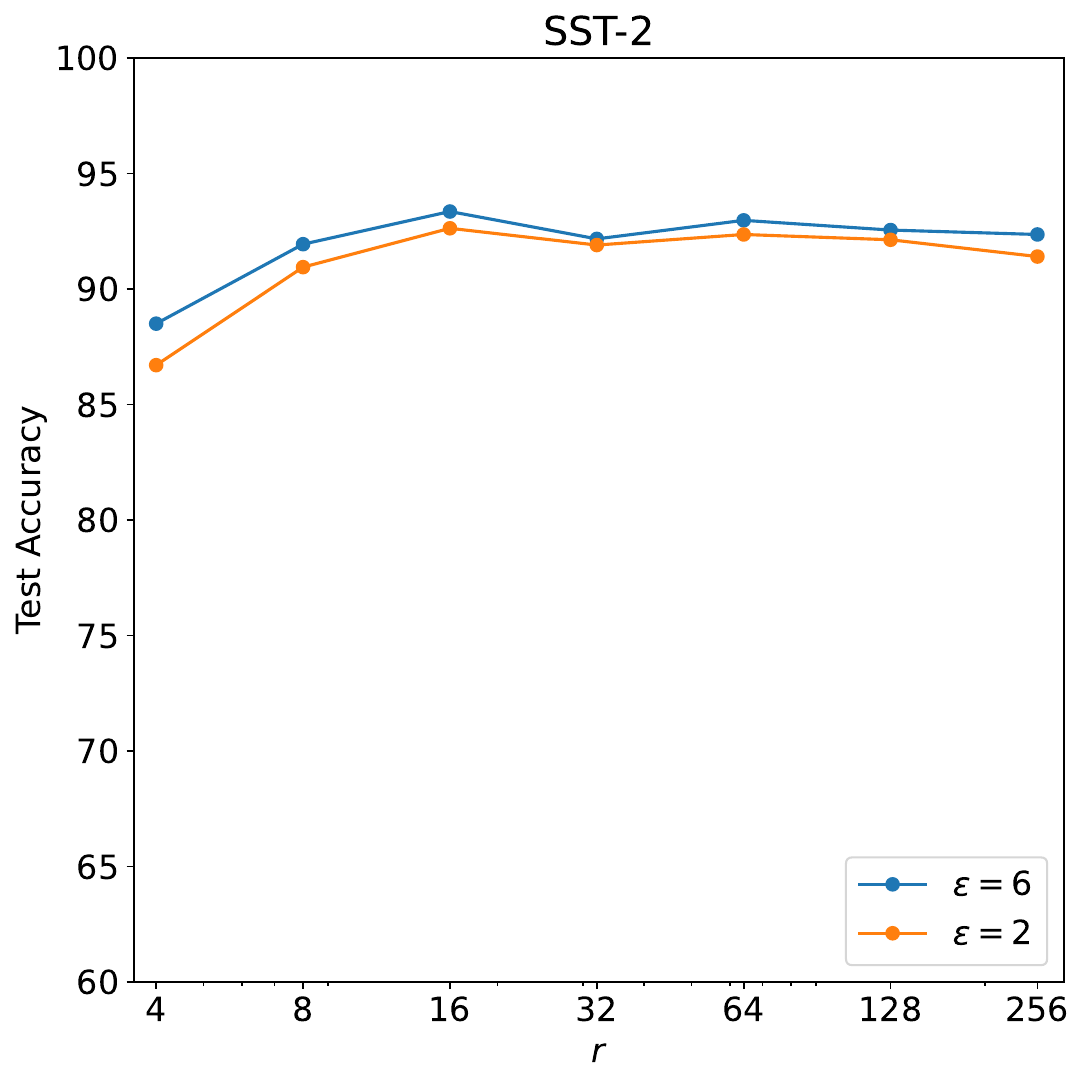}
    \end{minipage}
    \hfill
    \begin{minipage}{0.325\linewidth}
        \centering
        \includegraphics[width=\linewidth]{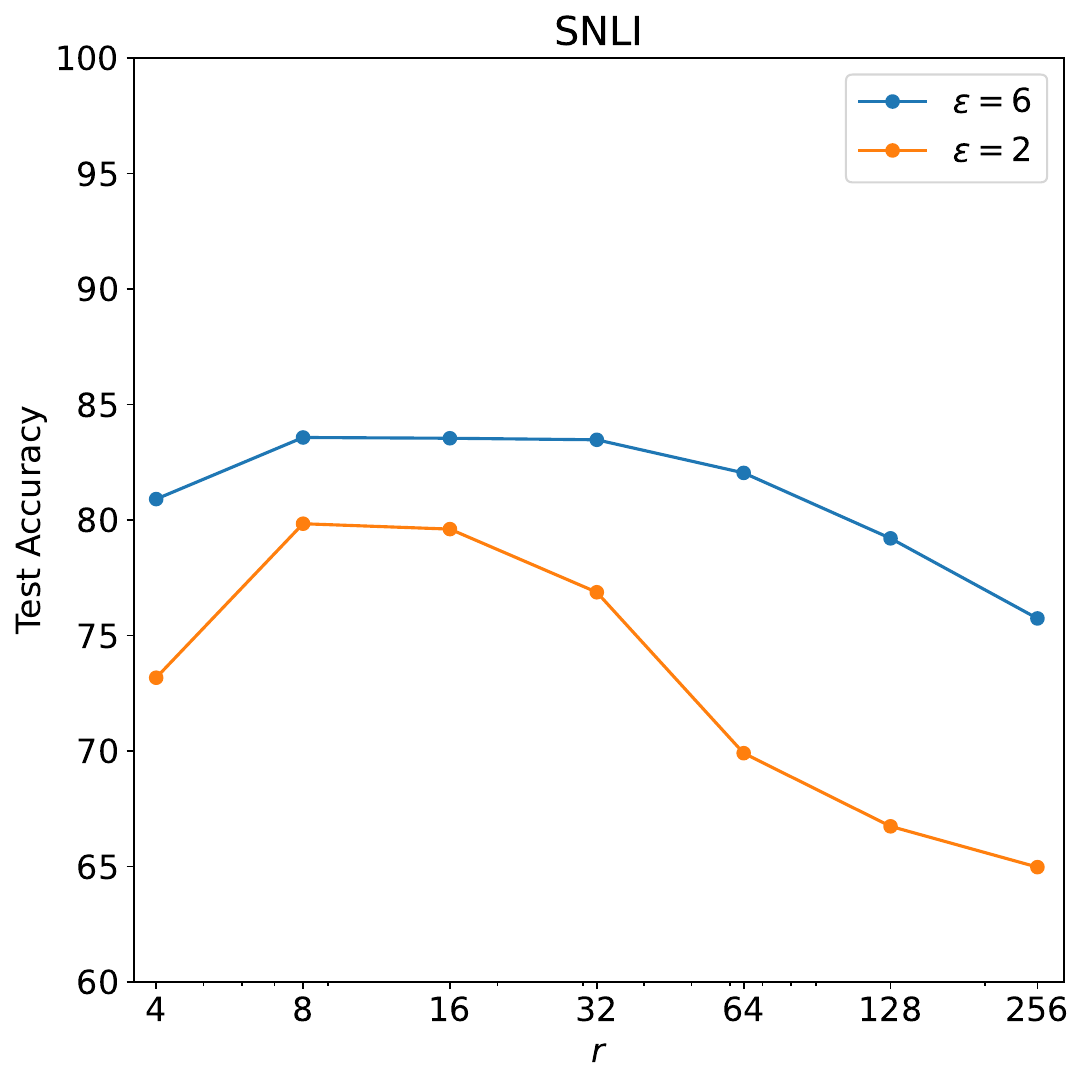}
    \end{minipage}
    \begin{minipage}{0.325\linewidth}
        \centering
        \includegraphics[width=\linewidth]{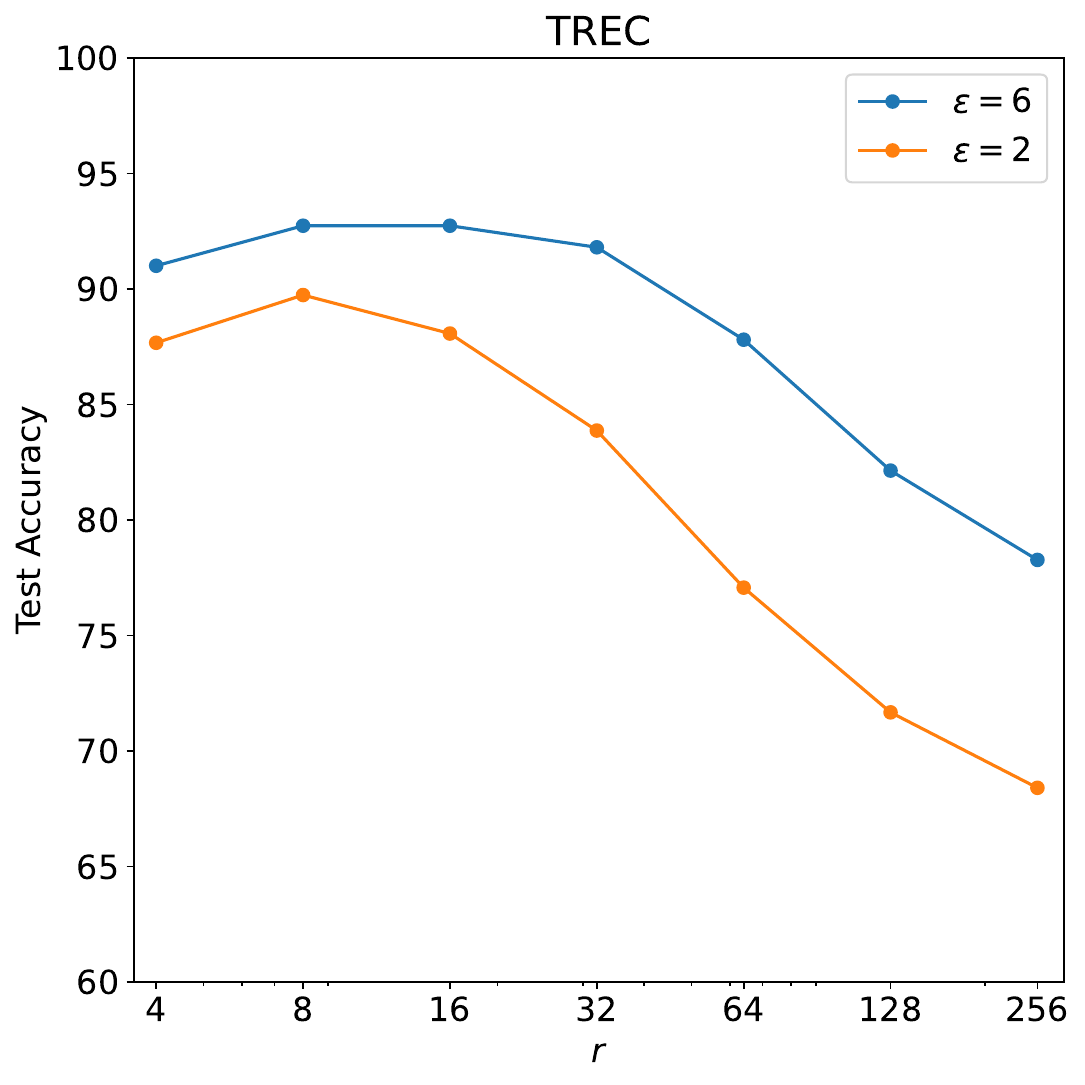}
    \end{minipage}
    \caption{Accuracy of RoBERTa-Large on SST-2, SNLI, and TREC using \ouralg{} with varying subspace dimension $r$ at $\varepsilon = 2$ $\varepsilon=6$ privacy levels.}
    \label{fig:roberta-r-ablation}
\end{figure}

To get the timing results shown in \cref{tab:roberta-timing}, we time how long the RoBERTa fine-tuning takes on SST-2 takes using an H100 GPU for DP-Adam, DPZero, and \ouralg{}, using the same experimental setup as we use to get the final results for different. The total batch size is set to 64. For DP-Adam, a batch size of 64 does not fit, so we use a physical batch size of 32 and gradient accumulation. For all methods, we fine-tune for 50 steps. The total train time for each method is inferred from the time for 50 steps, assuming 1000 total steps for DP-Adam and \ouralg{} and 10000 total steps for DPZero.

The memory and timing experiments were conducted on a single H100 GPU.

To generate the convergence plot shown in \cref{fig:roberta-convergence}, we fine-tune RoBERTa-Large on SST-2 using \ouralg{} and DPZero \citep{zhang2024dpzero}, and measure the development set accuracy every 50 steps. For \ouralg{}, we exactly match the experimental setup used to generate the results in \cref{tab:roberta-fewshot-results}. For DPZero, we use the same setup and implementation as given in the official GitHub implementation.  

\begin{table}[t]
\caption{\small Throughput and total training time for fine-tuning on SST-2 with RoBERTa-Large with a total batch size of 64 on an H100 GPU. Total training time is based on 1000 total steps for DP-Adam and \ouralg{} (the same total number of steps we use to generate the results in \cref{tab:roberta-fewshot-results}) and 10000 total steps for DPZero, (the number of steps reported to generate the final results in \citet{zhang2024dpzero}.}
\begin{center}
\begin{small}
\begin{tabular}{lp{2cm}p{2cm}}
\toprule
\textbf{Method} & \textbf{Throughput (Samples/s)}  & \textbf{Total Train Time (hours)} \\
\midrule
DP-Adam & 71.7 & 0.6 \\
DPZero & 268.1 & 1.7 \\
\ouralg{} & 75.9 & 0.6 \\
\bottomrule
\end{tabular}
\end{small}
\label{tab:roberta-timing}
\end{center}
\end{table}

\vskip 0.1in
\begin{figure}[ht]
    \centering
    \includegraphics[width=0.75\linewidth]{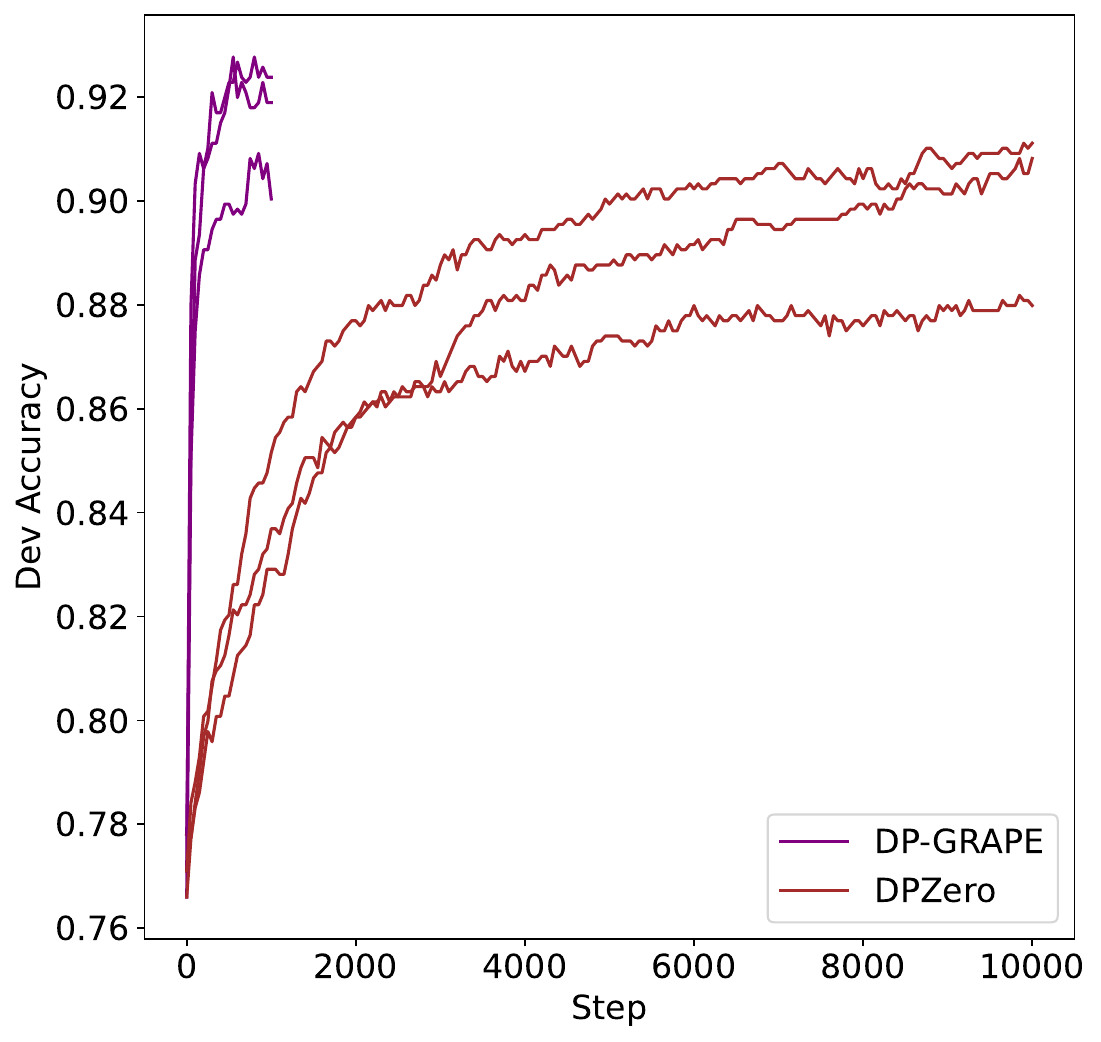}
    \caption{\small Convergence (as measured by development set accuracy) when fine-tuning RoBERTa-Large on SST-2 for \ouralg{} and DPZero, with runs for three different random seeds used to generate few-shot datasets shown.}
    \label{fig:roberta-convergence}
\end{figure}

\subsection{OPT Fine-Tuning} \label{app:experimental-details-opt}
For the OPT experiments, we also follow the same experimental setup and build off the same codebase as used by \citet{zhang2024dpzero} and \citet{malladi2023fine} to fine-tune OPT models with 1.3B, 2.7B, and 6.7B parameters on SST-2, BoolQ, and SQuAD, and DROP. We use a few-shot (1000 total training samples) setting for all datasets, and 1000 total samples for testing. Due to the increased computational requirements needing for fine-tuning as compared to the RoBERTa models, we search only for the best clipping parameter $C$ from the choices $\{ 0.1, 1.0, 5.0, 20.0\}$ for each model and dataset. \Cref{tab:opt-c} shows the $C$ we select for each model size and dataset for DP-Adam, and \ouralg{}. For \ouralg{}, we set $\eta=1\text{e}-4$, $F=100$, $T=2000$, and $r=16,32,64$ for the 1.3B, 2.7B, and 6.7B models, respectively. For DP-Adam, we set $\eta=1\text{e}-5$ and $T=3000$ after noting that a smaller learning rate and increased number of training steps is more stable. We train with a batch size of $8$ for all experiments, which may be achieved by gradient accumulation.

For the memory experiment, we train for 30 steps and record the maximum memory reserved by PyTorch \citep{paszke2019pytorch} using the \texttt{torch.cuda.max\_memory\_reserved()} function. For all sizes, we use a gradient accumulation step for the first-order methods so that accumulated gradients are included in the memory accounting.

The memory and timing experiments were conducted on a single H100 GPU.

For the timing experiment, we use the same experimental setup for each method as we use to generate the final results. We record the time taken to complete the first 30 steps of training, and use that to estimate the throughput (number of samples processed per second) and the total training time based on the total number of steps used to generate the final results listed in \cref{tab:opt-classification-results} and \cref{tab:opt-generation-results} (2000 steps for \ouralg{}, 3000 steps for DP-Adam, and 20,000 steps for DPZero, as listed in \citet{zhang2024dpzero}.
\vskip 0.3cm
{
\renewcommand{\arraystretch}{1.25}
\begin{table}[h!]
\caption{\small Clipping parameter $C$ selected for each 1.3B/2.7B/6.7B OPT models and different datasets, for DP-Adam and \ouralg{}. DP-Adam runs out of memory for the 6.7B models on all datasets.}
\vskip 0.05in
\centering
\begin{small}
\begin{tabular}{c|cccc}
\toprule
& \textbf{SST-2} & \textbf{BoolQ} & \textbf{SQuAD} & \textbf{DROP} \\
\midrule
DP-Adam & $20.0/20.0/-$ & $5.0/20.0/-$ & $5.0/5.0/-$ & $5.0/20.0/-$ \\
\ouralg{} & $20.0/20.0/1.0$ & $20.0/20.0/5.0$ & $0.1/0.1/0.1$ & $0.1/0.1/1.0$ \\
\bottomrule
\end{tabular}
\end{small}
\label{tab:opt-c}
\end{table}
}

\vskip 0.3in

\begin{table*}[!ht]
\caption{\small Throughput and total training time for fine-tuning on SQuAD with OPT models with a total batch size of 8 on an H100 GPU. Total training time is based on 2000 total steps for \ouralg{}, 3000 total steps for DP-Adam (both are the same total number of steps we use to generate the results in \cref{tab:opt-generation-results}) and 20000 total steps for DPZero, (the number of steps reported to generate the final results in \citet{zhang2024dpzero}. OOM indicates out of memory on an 80GB GPU with a batch size of 1 and gradient accumulation.}
\begin{center}
\begin{small}
\setlength{\tabcolsep}{4pt}
\begin{tabular}{lp{1.65cm}p{1.70cm}p{1.65cm}p{1.70cm}p{1.65cm}p{1.70cm}}
\toprule
\textbf{Model} & \multicolumn{2}{c}{OPT-1.3B} & \multicolumn{2}{c}{OPT-2.7B} & \multicolumn{2}{c}{OPT-6.7B} \\
 & \textbf{Throughput (Samples/s)}  & \textbf{Total Train Time (hours)} & \textbf{Throughput (Samples/s)}  & \textbf{Total Train Time (hours)} & \textbf{Throughput (Samples/s)}  & \textbf{Total Train Time (hours)} \\
\midrule
DPZero & 14.8 & 3.0 & 8.6 & 5.2 & 4.1 & 11.0 \\
DP-Adam & 9.6 & 0.7 & 4.4 & 1.5 & OOM & OOM \\
\ouralg{} & 10.7 & 0.4 & 5.6 & 0.8 & 2.5 & 1.8 \\
\bottomrule
\end{tabular}
\end{small}
\label{tab:opt-timing}
\end{center}
\end{table*}

\clearpage

\subsection{Hyperparameter Recommendations}
Although we performed somewhat extensive hyperparameter searches for our different experiments, in many cases, large searches may be computationally infeasible, so here we give general recommendations for selecting a learning rate, subspace dimension $r$, subspace change frequency $F$, and clipping parameter $C$ when training with \ouralg{}. For the learning rate, we found that a good range for fine-tuning with \ouralg{} was between 1e-4 and 1e-3 (slightly larger than DP-Adam). In general, based on our experiments and previous works which use gradient projection, practitioners should use a larger subspace dimension $r$ for pre-training than for fine-tuning, and they should use a larger $r$ for larger models. For example, for non-private GaLore \citep{zhao2024galore}, $r=128$ was used for pre-training a 60M Llama model, while $r=512$ was used for pre-training a 1B Llama model, and $r=4$ or $r=8$ was used for fine-tuning a RoBERTa-Base model [1]. We used a similar range for $r$ with respect to the different models sizes we applied \ouralg{} to. A subspace change frequency of $F=100$ worked well for all of the fine-tuning experiments. We recommend practitioners do a search over clipping values between roughly $C=0.1$ and $C=100$, as the best clipping parameter can vary for different tasks. 

\subsection{Computational Resources}  \label{app:experimental-details-compute}
We run all experiments on a single H100 GPU (although up to 4 were used at any one time to run separate experiments in parallel).

\section{Privacy and Convergence Analysis}

In this section, we provide a formal statement of \cref{th:informalconv} discussed in the main paper and its proof. Before the formal statement, we would like to give a more general version of the \ouralg{} algorithm discussed in the main paper. Instead of considering a partition of gradients in \ouralg{} which denotes each as gradients corresponding to a layer, we consider any general partition of the gradient vector in which we would independently project each part. This serves as a generalization as one can also introduce partitioning that is not necessarily demarcated by layers but may be something different (like the first column vector across layers). The generalized version of \ouralg{} is given by \cref{algo:d-free}.
\vskip 0.1in

\label{app:convergence}
\begin{algorithm}[h]
    \caption{Generalized Version of \ouralg}
    \begin{algorithmic}[1]
        \REQUIRE Dataset $X=\{\xi_1, \dots, \xi_n\}$, batch size $B$, number of blocks (layers) $L$ which is a partition $\indset_1, \indset_2, \cdots, \indset_L$ such that $\indset_1 \cup \indset_2 \cup \cdots \cup \indset_L = [d]$ $|\indset_\ell| = m_\ell n_\ell$ (hence $d = \sum_{\ell = 1}^Lm_\ell n_\ell$), initialization $w_0\in\bb R^d$, number of iterations $T$, stepsize $\eta>0$, clipping threshold $C>0$, privacy parameters $\eps>0, \delta\in(0,1)$.
        \STATE Compute privacy noise variance $\sigma = \frac{2C\sqrt{T\log(1/\delta)}}{n\epsilon}$
        \FOR{$t=0, 1, \cdots, T-1$}
            \STATE Sample B data-points $X^t = \{\xi^t_j\}_{j=1}^B$ uniformly from $X$.

            \FOR{$\ell = 1, \cdots, L$}
                \STATE For $i \in [r]$, sample $p^t_{\ell,i}$ i.i.d from  $\c N\paren{0, \frac{1}{r}I_{m_\ell}}$ and define $P^t_{\ell} = [p^t_{\ell,1} \cdots p^t_{\ell,r}] \in \bb R^{|\c U_\ell| \times r}$.
                \STATE For $j \in [B]$, compute projected gradient $R^t_{\ell,j} \gets {P^t_{\ell}}^\top (\nabla f(w_t; \xi_j^t) [\indset_\ell])$.
            \ENDFOR
            \STATE Define $R^t_j = [R^t_{1,j}, \cdots, R^t_{L,j}] \in \bb R^{r \times L}$ for each $j \in [B]$.
            \FOR{$\ell \in [L]$} 
                \STATE Privatize projected gradient of layer $\ell$, $\tilde{R}^t_\ell = \frac{1}{B}\sum^B_{j=1}\clip\left(R^t_{j}, {C}\right)[\cdot, \ell] + z^{t}_\ell,$ where $z^{t}_\ell\sim \cN(0, \sigma^2 \mI_r)$.
            \STATE Update parameters: $w_{t+1}[\indset_\ell] \;\gets\; w_t[\indset_\ell] - \eta P^t_\ell\tilde{R}^t_\ell$
            \ENDFOR
        \ENDFOR
        \STATE \textbf{Return} $w_\tau$ for $\tau$ sampled uniformly at random from $\{0,1,\cdots,T-1\}$.
    \end{algorithmic}
    \label{algo:d-free}
\end{algorithm} 
\vskip 0.1in
We now give a brief overview of how \cref{algo:d-free} works. We start with a fixed partition of the gradient vector, calling each partition a block. In \ouralg, each block corresponds to a layer. At each step, \cref{algo:d-free} samples a batch of $B$ data points and computes a block-wise random projection of gradients using $r$ Gaussian vectors. It then aggregates gradients across blocks, clips them, and adds Gaussian noise to ensure DP. Finally, the previously sampled $r$ vectors are used to map updates back to the $d$-dimensional space, and parameters are updated. This is repeated for $T$ rounds.

\cref{algo:d-free} is a more generalized version of \ouralg{}, where each $\c U_\ell$ for $\ell \in [L]$ represents indices in the flattened gradient vector corresponding to layer $\ell$. Thus, $\abs{\c U_\ell} = m_\ell n_\ell$. By replacing blocks with layer gradients, we see that projecting each block corresponds to projecting each layer’s gradient. However, \cref{algo:d-free} differs slightly from \ouralg{} by using $r$ different projections of the flattened layer-wise gradients instead of left/right projections of gradient matrices. The method in \ouralg{} has lower variance, leading to a similar upper bound on convergence. 

To facilitate our analysis, we make the following assumptions:

\begin{assumption}[Per-Sample Lipschitzness]
    The loss $f(\cdot;\xi)$ is $\lip$-Lipschitz for all $\xi \in X$, i.e. for all $w_1, w_2 \in \bb R^d$ $\norm{f(w_1;\xi) - f(w_2;\xi)} \leq \lip\norm{w_1 - w_2}$.
    \label{asp:lip}
\end{assumption}

\begin{assumption}[Smoothness]
    The average loss $F(w) := \frac{1}{n} \sum_{i=1}^n f(w;\xi_i)$ is $\smo$-smooth for every given dataset $X$ i.e. for all $w_1, w_2 \in \bb R^d$, $\norm{\nabla F(w_1) - \nabla F(w_2)} \leq \smo \norm{w_1 - w_2}$.
    \label{asp:smoothness}
\end{assumption}

\begin{assumption}[Finiteness of Optimal Value]
    $F^*:=\min_{w\in\bb R^d} F(w)$ is finite.
    \label{asp:finitemin}
\end{assumption}

\remark{It is important to note that these assumptions are standard in the analysis of private non-convex optimization \citep{luw24privnonconvex, zhang2024dpzero}.}

Before presenting the complete proof of our theorem, we would like to define the following notations for the ease of stating our proofs.

\textbf{Notations and Lemmas:} For a set $A \subseteq X$, we define $f(w;A) = \frac{1}{|A|}\sum_{\xi \in A} f(w; \xi)$. By this definition, $F(w) = f(w; X)$. Moreover, for a vector $a \in \bb R^d$ and an ordered set $P = \{p_1, p_2, \cdots, p_q\} \subseteq [d]$, we define $a[P] := (a[p_1], a[p_2], \cdots a[p_q]) \in \bb R^q$ where $a[i]$ represents the $i^{th}$ index of the vector $a$. We assume that indexing starts from 1. $\norm{\cdot}$ represents the $\ell_2$ norm while $\norm{\cdot}_F$ represents the Frobenius norm. We now state some lemmas which would be useful throughout the proof.

\begin{lemma}
\label{lem: posrand}
Consider any random variable $X \geq 0$ and an event $Q$, then we have that
\begin{align*}
    \bb E[X|Q] \leq \frac{\bb E [X]}{\bb P (Q)}
\end{align*}
\end{lemma}
\begin{proof}
    This directly follows from the law of total probability and the non-negativity of $X$
    \begin{align*}
        \bb E [X] = \bb E[X | Q] \bb P (Q) + \bb E[X | Q^c] \bb P (Q^c) \geq \bb E[X | Q] \bb P (Q)
    \end{align*}
    which proves the given claim.
\end{proof}

\begin{lemma}[\citet{greene2003econometric}]
\label{lem: truncgauss}
Consider $X \sim \c N(0, \sigma^2)$ then we have that
\begin{align*}
    \bb E[X^2 | X \geq K] = \sigma^2\left[1 + \frac{K}{\bb P[X \geq K]\sigma\sqrt{2\pi}} e^{-\frac{K^2}{2\sigma^2}}\right]
\end{align*}
\end{lemma}

\begin{lemma}
    \label{lem:condexptransfer}
    Consider any random variable $X$ and a set $Q$ such that $\{X: X \geq t\} \subset Q$, then we get that
    \begin{align*}
        \bb E[X| Q] \leq \bb E[X | X \geq t]
    \end{align*}
    \begin{proof}
        Using law of total probability, we get 
        \begin{align*}
            \bb E[X|Q] = \bb E[X|Q, X \geq t] \bb{P} (X \geq t | Q) + \bb E[X|Q, X < t] \bb{P} (X < t | Q).
        \end{align*}
        Since, $\{X: X \geq t\} \subset Q$, we have that $\bb E[X|Q, X \geq t] = \bb E[X|X \geq t]$. By conditioning, we have that $E[X|Q, X < t] < t \leq E[X| X \geq t]$ proving the given claim. 
    \end{proof}
\end{lemma}

\begin{lemma}[\citet{lin2020gradient}]
\label{lem: lei}
Let $\{a_l\}_{l \in [n]}$ be an arbitrary collection of vectors such that $\sum_{l=1}^{n} a_l = 0$. Further, let $\mathcal{S}$ be a uniformly random subset of $[n]$ of size $m$. Then,\[
\mathbb{E}\left\|\frac{1}{m} \sum_{l \in \mathcal{S}} a_l \right\|^2 = \frac{n - m}{(n - 1) m} \frac{1}{n}\sum_{l=1}^{n} \|a_l\|^2 \leq \frac{{\ind}_{\{m < n\}}}{m~n}\sum_{l=1}^{n}\|a_l\|^2.
\]
\end{lemma}

\begin{lemma}[\citet{zhang2024dpzero}]
    Let $u, v$ be uniformly sampled from the standard $d$-dimensional Gaussian, let $a\in\bb R^d$ be some fixed vector independent of $u$, and $H\in\bb R^{d\times d}$ be some fixed matrix independent of $u$. We have that
    \begin{itemize}
        \item[$(i)$] $\bb E[u]=0\,$ and $\,\bb E[uu^\top]=\mI_d$.
        
        \item[$(ii)$] $\bb E_u[(u^\top a) u] = a\,$ and $\bb E_u[(u^\top a)^2 \norm{u}^2] = (d+2)\norm{a}^2$.

        \item[$(iii)$] $\bb E_u[u^\top H u] = \mathrm{Tr}(H)$.
    \end{itemize}
    \label{lm:sphere}
\end{lemma}

\begin{lemma}[Theorem 1 of \citet{abadi2016deep}]
    There exist constants $c_1$ and $c_2$ so that given the number of steps T, batch size B, and sensitivity $\Delta$, for any $\eps < c_1 \frac{B^2}{n^2}T$, the Gaussian Mechanism with noise level $\sigma$ applied for $T$ steps is $(\epsilon, \delta)$-differentially private for any $\delta > 0$ if we choose
    \begin{align*}
        \sigma \geq c_2\frac{B\Delta\sqrt{T\log(1/\delta)}}{n\epsilon}.
    \end{align*}
    \label{lm:subsampledmomentsaccountant}
\end{lemma}

\paragraph{Privacy-Utility Tradeoff:} Now, we state the main convergence result:

\begin{theorem}[Formal version of \cref{th:informalconv}]
    \label{th:formalconv}
    For any $\eps>0$ and $\delta\in(0,1)$, Algorithm \ref{algo:d-free} is $(\eps, \delta)$-DP. Under Assumptions \ref{asp:lip}, \ref{asp:smoothness}, \ref{asp:finitemin}, $r \leq d - 1$ and the fact that $\max_{0\leq t\leq T}\abs{F(w_t) - F^*}\leq D$, there exist a set of parameters such that the output $\hat{w}$ satisfies 
    \begin{align*}
        \bb E\norm{\nabla F(\hat{w})}^2 \leq \left(\smo D + \paren{2 + \mathcal M + 2\log\paren{\left(10 + \frac{64dD\smo}{r\lip^2}\right) \frac{n^3 \eps^2}{2\log(1/\delta)}}}\lip^2\right)\frac{4\sqrt{2d\log(1/\delta)}}{n\eps},
    \end{align*}
    where $\mathcal M = \tilde{O}(1)$.
\end{theorem}

\begin{proof}

\textbf{Privacy Guarantee.} Assuming that the adversary has the knowledge of the projection matrices, the proof of privacy follows by the fact that the sensitivity when one data point is replaced is given by $\Delta = \frac{2C}{B}$ (Lines 8 and 10 in Algorithm~\ref{algo:d-free}). Then we use \cref{lm:subsampledmomentsaccountant} with a sampling factor of $\frac{B}{n}$ and explicitly get the constants from \citet{bassily2019private} to get the value of $\sigma = \frac{2C\sqrt{T\log(1/\delta)}}{n\epsilon}$ in the algorithm for $\eps \leq \frac{2B^2T}{n^2}$. Since, we assumed the matrices to be provided to the adversary, the projection of the noisy projected gradients back to the parameter space can be treated as a post-processing and hence DP is preserved due to post processing.

\begin{remark}
    Note that in the above definition, we use the replace-one notion of sensitivity. But, the sensitivity for add-one/remove-one notion of DP still remains the same as the  that we have mentioned above. For completion, we provide a short proof. Consider the per-sample vectors to be $v_1, \cdots, v_B$. Then, the sensitivity of the add-one/remove-one notion of DP would be $\Delta = \left|\frac{1}{n+1} \sum_{i=1}^{n+1} v_i-\frac{1}{n} \sum_{i=1}^n v_i\right|=\left|\frac{v_{n+1}}{n+1}-\frac{1}{n(n+1)} \sum_{j=1}^n v_j\right| \leq \frac{2C}{n}$. Thus, we get that the sensitivity upper bound for the replace-one and add/remove one notion of DP is exactly the same.
\end{remark}

\begin{remark}
    Note that we are essentially using a sub-sampled Gaussian mechanism in our algorithm which also satisfies other notions of DP such as Truncated Concentrated DP \citep{bun2018composable} and Gaussian DP \citep{dong2022gaussian}. Hence, we can convert the bounds in terms of the parameters for different notions of privacy by getting the variance of the Gaussian noise in those parameters.
\end{remark}

\paragraph{Utility guarantee.}
We focus on the utility guarantee on $\bb E\norm{\nabla F(w_\tau)}^2$. Before going into the details of the proof, we would like to define some notations which would make things easier for us in the proof. For $A \subseteq X$ Define $\nabla f(w;A) [U_\ell] = \nabla_\ell f(w;A)$ and w.l.o.g. take $\nabla f(w, A) = (\nabla_1 f(w;A)^T, \nabla_2 f(w;A)^T, \cdots, \nabla_L f(w;A)^T)^T \in \bb R^{d}$. The same can be extended for $F(w) = f(w;X)$.

For any $t \in \{0, \cdots, T-1\}, i \in [r], j \in [n]$, and $\ell \in [L]$ consider the term $\paren{\paren{p^t_{\ell, i}}^\top \nabla_\ell f(w_t;\xi_j)}^2$. Since $p^t_{\ell, i} \sim \c N\paren{0, \frac{1}{r} I_{m_\ell}}$, 
we have that $\paren{\paren{p^t_{\ell, i}}^\top \nabla_\ell f(w_t;\xi_j)}^2 = \frac{\norm{\nabla_\ell f(w_t;\xi_j)}^2}{r} \paren{V^t_{\ell,i,j}}^2$ 
where $V^t_{\ell,i,j} \sim \c N(0, 1)$. Let $Q^t_{\ell,i,j}$ be the event such that $\abs{V^t_{\ell,i,j}} \leq \frac{C}{\lip}$. Let $Q_t = \bigcap_{i=1}^r\bigcap_{j=1}^n\bigcap_{\ell = 1}^L Q^t_{\ell,i,j}$. Thus, the probability that clipping does not happen at one iteration is greater than $\bb P(Q_t)$.
We also denote $Q = \bigcap_{t=0}^{T-1} Q_t$. Hence, $\bar Q = \bigcup_{t=0}^{T-1}\bigcup_{i=1}^r\bigcup_{j=1}^n\bigcup_{\ell = 1}^L \Bar{Q}^t_{\ell,i,j}$ and for some $V \sim \c N(0, 1)$ we have that $\bb P\left[\abs{V^t_{\ell,i,j}} \geq \frac{C}{\lip}\right] = \bb P\left[\abs{V} \geq \frac{C}{\lip}\right]$ for some $V \sim \c N(0, 1)$. By the union bound, we have that
\begin{align*}
    \bb P(\bar Q) \leq TLnr\bb P\left[\abs{V} \geq \frac{C}{\lip} \right].
\end{align*}

To simplify the notation, we let $G(x_t)$ represent 
\begin{align*}
    G_\ell(w_t)
    & =
    \frac{1}{n} \sum_{j=1}^n \sum_{i=1}^r  p^t_{\ell,i} \paren{p^t_{\ell,i}}^\top  \nabla_\ell f(w_t;\xi_j) =
    \sum_{i=1}^r p^t_{\ell,i} \paren{p^t_{\ell,i}}^\top \nabla_\ell F(w_t),
\end{align*}

For all $\ell \in [L]$, let $\hat{G}_\ell(w_t; X_t) = \frac{1}{B} P^t_\ell\paren{\sum_{j=1}^B  \clip\left(R^t_{j}, {C}\right)[\cdot, \ell]}$ and let
\begin{align*}
    G_\ell(w_t; X_t)
    & =
    \frac{1}{B} \sum_{j=1}^B\sum_{i=1}^r p^t_{\ell,i} \paren{p^t_{\ell,i}}^\top \nabla_\ell f(w_t;x^t_j).
\end{align*}

Note that the definition of $G_\ell$ is agnostic to whether clipping happens of not. But, conditioned on the event that clipping does not happen ($Q$), we have that $\hat{G}_\ell(w_t; X_t) = G_\ell(w_t; X_t)$.

Algorithm \ref{algo:d-free} becomes $X_t \sim Unif(X)$, $w_{t+1}[U_\ell] = w_t[U_\ell] - \eta (\hat{G}_\ell(w_t; X_t) + P^t_\ell z^{t}_\ell)$ under the above notation. Let $\Hat{G}(w_t; X_t) = (\Hat{G}_1(w_t; X_t)^\top, \cdots, \Hat{G}_L(w_t; X_t)^\top)^\top$, $P_t = \left\{P^t_1, \cdots, P^t_L\right\}$, and $z_t = \left\{z^t_1, \cdots, z^t_L\right\}$. Define $N_t = \left[\paren{P^t_1z^t_1}^\top, \cdots, \paren{P^t_Lz^t_L}^\top\right]^\top$
. Then, by smoothness of $F(w)$, we get that
\begin{align*}
    F(w_{t+1}) &\leq F(w_t) - \eta \nabla F(w_t)^\top(\hat{G}(w_t;X_t) + N_t) + \frac{\eta^2\smo}{2}\norm{\hat{G}(w_t;X_t) + N_t}^2\\
    & = F(w_t) - \eta \sum_{\ell = 1}^L \nabla_\ell F(w_t)^\top \paren{\hat{G}_\ell(w_t;X_t) + P^t_\ell z^t_\ell} + \frac{\eta^2\smo}{2} \sum_{\ell = 1}^L\norm{\hat{G}_\ell(w_t;X_t) + P^t_\ell z^t_\ell}^2.
\end{align*}

The event $Q_t$ depends on the randomness in $P_{<(t+1)}:=\{P_0, \cdots, P_t\}$, $X_{< (t+1)} = \{X^0, \cdots, X^t\}$ and $z_{<t}:=\{z_0, z_1, \cdots, z_{t-1}\}$. Note that the noise $z_t$ sampled from $\c N(0, \sigma^2I_r)$ is independent of $P_{<(t+1)}$, $X_{<(t+1)}$, $z_{<t}$, $w_t$, and the dataset $X$. Given that event $Q_t$ happens, it implies that we do not clip which implies that $\hat{G}(w_t; X_t) = G(w_t; X_t)$. Combining this with the fact that we clip per-sample, we get $\bb E_{X_{< (t+1)}}[\hat{G}(w_t; X_t)|Q_t] = \bb E_{X_{< (t+1)}}[G(w_t; X_t)|Q_t] = \bb E_{X_{< t}}[G(w_t)|Q_t]$. Conditioned on the event $Q_t$ and taking expectation with respect to $z_{<(t+1)}$, $X_{<(t+1)}$ and $P_{<(t+1)}$, we have that
\begin{equation}
    \begin{split}
        \bb E_{z_{<(t+1)}, X_{<(t+1)}, P_{<(t+1)}}[F(w_{t+1}) | Q_t]
        &\leq
        \bb E_{z_{<(t+1)}, X_{<(t+1)}, P_{<(t+1)}} [F(w_{t}) | Q_t] -\\ 
        & \eta\bb E_{z_{<(t+1)}, X_{<(t+1)}, P_{<(t+1)}}\left[\sum_{\ell = 1}^L \nabla_\ell F(w_t)^\top \paren{\hat{G}_\ell(w_t;X_t) + P^t_\ell z^t_\ell} \middle| Q_t \right] \\
        &+
        \frac{\eta^2\smo}{2} \bb E_{z_{<(t+1)}, X_{<(t+1)}, P_{<(t+1)}} \left[\sum_{\ell = 1}^L\norm{\hat{G}_\ell(w_t;X_t) + P^t_\ell z^t_\ell}^2\middle| Q_t\right]\\
        &= \bb E_{z_{<t}, X_{<t}, P_{<t}} [F_S(x_t) | Q_t] \\
        &- \eta\underbrace{\sum_{\ell = 1}^L \bb E_{z_{<t}, X_{<t}, P_{<(t+1)}} \left[\nabla_\ell F(w_t)^\top G_\ell(w_t) \,\middle|\, Q_t \right]}_{\encircled{1}}\\
        &+ \frac{\eta^2\smo}{2}\underbrace{\sum_{\ell = 1}^{L}\bb E_{z_{<t}, X_{<(t+1)}, P_{<(t+1)}}\left[\norm{G_\ell(w_t; X_t)}^2\middle| Q_t\right]}_{\encircled{2}} \\
        &+ \frac{\eta^2\smo}{2} \underbrace{\sum_{\ell = 1}^L\bb E_{z_{<(t+1)}, X_{<t}, P_{<(t+1)}}\left[{z^t_\ell}^\top {P_\ell^t}^\top P_\ell^tz^t_\ell\middle| Q_t\right]}_{\encircled{3}}.
    \end{split}
    \label{eq:main-smooth}
\end{equation}

For term $\encircled{1}$, we get that
\begin{equation*}
    \encircled{1} =
    \sum_{\ell = 1}^L \bb E_{z_{<t}, X_{<t}, P_{<(t+1)}} \left[\nabla_\ell F(w_t)^\top G_\ell(w_t) \,\middle|\, Q_t \right].
\end{equation*}
By the law of total probability and \cref{lm:sphere}, since $P_t$ is independent of $w_t$, for each $\ell \in [L]$ we know that
\begin{align*}
    \bb E_{z_{<t}, X_{<t}, P_{<(t+1)}}\left[\nabla_\ell F(w_t)^\top G_\ell(w_t) \,\middle|\, Q_t \right] \bb P(Q_t) + \bb E_{z_{<t}, X_{<t}, P_{<(t+1)}}\left[\nabla_\ell F(w_t)^\top G_\ell(w_t) \,\middle|\, \bar Q_t \right] \bb P(\bar Q_t)\\
     = \bb E_{z_{<t}, X_{<t}, P_{<(t+1)}}\left[\nabla_\ell F(w_t)^\top G_\ell(w_t) \right] =
    \bb E_{z_{<t}, X_{<t}, P_{<t}}\left[\norm{\nabla_\ell F(w_t)}^2 \right],
\end{align*}

Rearranging terms, we thus obtain that
\begin{equation}
    \label{ineq:innprod}
    \begin{split}
        \bb E_{z_{<t}, X_{<t}, P_{<(t+1)}}\left[\nabla_\ell F(w_t)^\top G_\ell(w_t) \,\middle|\, Q_t \right]
        & =
        \frac{\bb E_{z_{<t}, X_{<t}, P_{<t}}\norm{\nabla_\ell F(w_t)}^2}{\bb P(Q_t)} - \\
        &\frac{\bb E_{z_{<t}, X_{<t}, P_{<(t+1)}} \left[\nabla_\ell F(w_t)^\top G_\ell(w_t) \,\middle|\, \bar Q_t \right] \bb P(\bar Q_t)}{\bb P(Q_t)}
    \end{split}
\end{equation}.

Using the definition of our event $\Bar{Q_t} = \bigcup_{i=1}^r\bigcup_{j=1}^n \bigcup_{\ell = 1}^L \bar{Q}^t_{\ell,i,j}$, we have that $\bar{Q}^t_{\ell,i,j} \subset \Bar{Q_t}$. Thus, we have that,
\begin{equation}
    \label{ineq:innproduppfull}
    \begin{split}
        \bb E_{z_{<t}, X_{<t}, P_{<(t+1)}} \left[\nabla_\ell F(w_t)^\top G_\ell(w_t) \middle | \Bar{Q_t}\right] &= \sum_{i=1}^r \bb E_{z_{<t}, X_{<t}, P_{<(t+1)}} \left[(\paren{p^t_{\ell, i}}^\top \nabla_\ell F(w_t))^2 \middle | \Bar{Q_t} \right] \\
    &\leq \frac{1}{n} \sum_{j=1}^n \sum_{i=1}^r \bb E_{z_{<t}, X_{<t}, P_{<(t+1)}} \left[(\paren{p^t_{\ell, i}}^\top \nabla_\ell f(w_t; \xi_j))^2 \middle | \Bar{Q_t} \right] \\
    &= \frac{1}{nr} \sum_{j=1}^n \sum_{i=1}^r \norm{\nabla_\ell f(w_t;\xi_j)}^2 \bb E_{z_{<t}, X_{<t}, P_{<(t+1)}} \left[\paren{V^t_{\ell,i,j}}^2 \middle | \Bar{Q_t} \right] \\
    &\leq \frac{1}{rn} \sum_{j=1}^n  \sum_{i=1}^r \norm{\nabla_\ell f(w_t;\xi_j)}^2 \bb E\left[\paren{V^t_{\ell,i,j}}^2 \middle | \bar{Q}^t_{\ell,i,j} \right] \\
    &= \frac{1}{rn} \sum_{j=1}^n \sum_{i=1}^r \norm{\nabla_\ell f(w_t;\xi_j)}^2\bb E\left[\paren{V^t_{\ell,i,j}}^2 \middle | \abs{V^t_{\ell,i,j}} \geq \frac{C}{\lip} \right] \\
    &= \frac{1}{n} \sum_{j=1}^n \norm{\nabla_\ell f(w_t;\xi_j)}^2 \left[1 + \frac{C}{\bb P\left[\abs{V} \geq \frac{C}{\lip}\right]\lip \sqrt{2\pi}}e^{-\frac{C^2}{2\lip^2}}\right].
    \end{split}
\end{equation}

The first equality uses the definition of $G_\ell$, the second inequality is Young's Inequality, the third equality uses the fact that $a^\top Z \sim \mathcal N(0, \norm{a}^2)$ for $Z \sim \mathcal N(0, I_d)$ and $a \in \mathbb R^d$ independent of $Z$. The fourth inequality uses Lemma~\ref{lem:condexptransfer} with X as $\paren{V^t_{\ell,i,j}}^2$. The fifth equality uses the fact that $\paren{V^t_{\ell,i,j}}^2 \geq \frac{C^2}{L\lip^2}$ is equivalent to $\abs{V^t_{\ell,i,j}} \geq \frac{C}{\lip}$. The sixth equality used Lemma~\ref{lem: truncgauss}. Combining~\ref{ineq:innproduppfull} with~\ref{ineq:innprod}, we obtain that
\begin{equation*}
    \begin{split}
        E_{z_{<t}, X_{<t}, P_{<(t+1)}} \left[\nabla_\ell F(w_t)^\top G_\ell(w_t) \,\middle|\, Q_t \right] &\geq \frac{\bb E_{z_{<t}, X_{<t}, P_{<t}}\norm{\nabla_\ell F(w_t)}^2}{2\,\bb P(Q_t)} \\
        - \frac{1}{n} \sum_{j=1}^n \frac{\norm{\nabla_\ell f(w_t;\xi_j)}^2\bb P(\bar Q_t)}{\bb P(Q_t)}&\left[1 + \frac{C}{\bb P\left[\abs{V} \geq \frac{C}{\lip}\right]\lip \sqrt{2\pi}}e^{-\frac{C^2}{2\lip^2}}\right].
    \end{split}
\end{equation*}

Thus, taking the sum over all $\ell \in [L]$, we get that
\begin{equation}
    \begin{split}
        \encircled{1} &\geq \sum_{\ell = 1}^L \frac{\bb E_{z_{<t}, X_{<t}, P_{<t}}\norm{\nabla_\ell F(w_t)}^2}{2\,\bb P(Q_t)} - \frac{1}{n} \sum_{j=1}^n  \sum_{\ell = 1}^L \frac{\norm{\nabla_\ell f(w_t;\xi_j)}^2\bb P(\bar Q_t)}{\bb P(Q_t)} \left[1 + \frac{C}{\bb P\left[\abs{V} \geq \frac{C}{\lip}\right]\lip \sqrt{2\pi}}e^{-\frac{C^2}{2\lip^2}}\right]\\
        &= \frac{\bb E_{z_{<t}, X_{<t}, P_{<t}}\norm{\nabla F(w_t)}^2}{2\,\bb P(Q_t)} - \frac{1}{n} \sum_{j=1}^n  \frac{\norm{\nabla f(w_t;\xi_j)}^2\bb P(\bar Q_t)}{\bb P(Q_t)}\left[1 + \frac{C}{\bb P\left[\abs{V} \geq \frac{C}{\lip}\right]\lip \sqrt{2\pi}}e^{-\frac{C^2}{2\lip^2}}\right]\\
        &\geq \frac{\bb E_{z_{<t}, X_{<t}, P_{<t}}\norm{\nabla F(w_t)}^2}{2\,\bb P(Q_t)}
        - \frac{\lip^2\bb P(\bar Q_t)}{\bb P(Q_t)}\left[1 + \frac{C}{\bb P\left[\abs{V} \geq \frac{C}{\lip}\right]\lip \sqrt{2\pi}}e^{-\frac{C^2}{2\lip^2}}\right]
    \end{split}
    \label{eq:inner-product}
\end{equation}

The second equality comes from the fact that the full gradient of the model is just concatenated version of the layer wise gradients. The third inequality follows from the per-sample Lipschitzness assumption~\ref{asp:lip}. 

For term $\encircled{2}$, 
using \cref{lem: posrand}, we have 
\begin{align*}
        \encircled{2}&\leq \sum_{\ell = 1}^L \frac{\bb E_{z_{<t}, X_{<(t+1)}, P_{<(t+1)}}\left[\norm{G_\ell(w_t; X_t)}^2\right]}{\bb P(Q_t)}
\end{align*}

\begin{equation*}
    \begin{split}  
        \bb E_{z_{<t}, X_{<(t+1)}, P_{<(t+1)}}\left[\norm{G_\ell(w_t; X_t)}^2\right] &\leq 2 \E_{z_{<t}, X_{<t}, P_{<(t+1)}} [\norm{G_\ell(w_t)}^2] + \\
        &2\E_{z_{<t}, X_{<(t+1)}, P_{<(t+1)}}[\norm{G_\ell(w_t;X_t) - G_\ell(w_t)}^2] \\
        &= 2\bb E_{z_{<t}, X_{<t}, P_{<(t+1)}}\left[\norm{\sum_{i=1}^r \paren{p^t_{\ell, i}}^T\nabla_\ell F(w_t) p^t_{\ell, i}}^2\right] + \\
         2\bb E_{z_{<t}, X_{<(t+1)}, P_{<(t+1)}}& \left[\norm{\sum_{i=1}^r \paren{p^t_{\ell, i}}^T(\nabla_\ell f(w_t; X_t) - \nabla_\ell F(w_t)) p^t_{\ell, i}}^2\right] \\
        &= \paren{\frac{d + r + 1}{r}} \bb E_{z_{<t}, X_{<t}, P_{<t}}[\norm{\nabla_\ell F(w_t)}^2] + \\
        & \paren{\frac{d + r + 1}{r}} \bb E_{z_{<t}, X_{<(t+1)}, P_{<t}}[\norm{\nabla_\ell f(w_t; X_t) - \nabla_\ell F(w_t)}^2] \\
        & \leq \frac{2d}{r} \bb E_{z_{<t}, X_{<t}, P_{<t}}[\norm{\nabla_\ell F(w_t)}^2] + \frac{2d \norm{\nabla_\ell F(w_t)}^2 {\ind}_{\{B < n\}}}{rB}.
    \end{split}
\end{equation*}

Thus, taking the sum and using the fact that $\sum_{\ell = 1}^L \norm{\nabla_\ell F(w_t)}^2 = \norm{\nabla F(w_t)}^2 \leq \lip^2$ and $r \leq d - 1$, we get that
\begin{equation}
    \encircled{2}\leq \frac{2d}{r\bb P(Q_t)} \bb E_{z_{<t}, X_{<t}, P_{<t}}[\norm{\nabla F(w_t)}^2] + \frac{8d\lip^2 {\ind}_{\{B < n\}}}{rB \bb P(Q_t)}
    \label{eq:grad-quad}
\end{equation}

For term $\encircled{3}$, $P^t_\ell$ is essentially a $|\c U_\ell| \times r$ matrix of independent $\c N\paren{0, \frac{1}{r}}$ entries. Since the identity matrix is positive semi-definite, using \cref{lem: posrand} and \cref{lm:sphere} we get that 
\begin{equation}
    \begin{split}
         \encircled{3} &= \sum_{\ell = 1}^L\bb E_{z_{<(t+1)}, X_{<t}, P_{<(t+1)}}\left[{z^t_\ell}^\top {P_\ell^t}^\top P_\ell^t z^t_\ell\middle| Q_t\right] \\
         & \leq
        \frac{\sum_{\ell = 1}^L\bb E_{z_{<(t+1)}, X_{<t}, P_{<(t+1)}}\left[{z^t_\ell}^\top {P_\ell^t}^\top P_\ell^tz^t_\ell\right]}{\bb P(Q_t)} \\
        &= \frac{\sigma^2 \sum_{\ell = 1}^L \bb E_{z_{<t}, X_{<t}, P_{<(t+1)}}\left[Tr({P_\ell^t}^\top P_\ell^t)\right]}{\bb P(Q_t)} \\
        &= \frac{\sigma^2 \sum_{\ell = 1}^L \bb E_{z_{<t}, X_{<t}, P_{<(t+1)}} \left[ \sum_{i=1}^r \norm{p^t_{\ell,i}}^2 \right]}{\bb P(Q_t)} \\
        &= \frac{\sigma^2 \sum_{\ell = 1}^L |\cal U_\ell|}{\bb P(Q_t)} = \frac{\sigma^2 d}{\bb P(Q_t)}.
    \end{split}
    \label{eq:noise-quad}
\end{equation}

Plugging \eqref{eq:inner-product}, \eqref{eq:grad-quad} and \eqref{eq:noise-quad} back into \eqref{eq:main-smooth}, we obtain that
\begin{align*}
    \bb E_{z_{<(t+1)}, X_{<(t+1)}, P_{<(t+1)}}[F(w_{t+1})|Q_t]
    & \leq
    \bb E_{z_{<t}, X_{<t}, P_{<t}}[F(w_t)|Q_t] \\
    &- \frac{\eta}{2}\left(1 - \frac{2d\smo\eta}{r}\right)\frac{\bb E_{z_{<t}, \xi_{<t}, U_{<t}}\norm{\nabla F(w_t)}^2}{\bb P(Q_t)} \\
    &+ \frac{\eta \lip^2\bb P(\bar Q_t)}{\bb P(Q_t)}\left[1 + \frac{C}{\bb P\left[\abs{V} \geq \frac{C}{\lip}\right]\lip \sqrt{2\pi}}e^{-\frac{C^2}{2\lip^2}}\right]\\
    &+ \frac{4d\eta^2 \lip^2 \smo \ind_{m < n}}{Br \bb P(Q_t)} + \frac{\eta^2 \smo \sigma^2 d}{2\bb P(Q_t)}
    \label{eq:onestep3}
\end{align*}

Assuming $\frac{2d\lambda \eta}{r} \leq 1$ and choosing $\eta \leq \frac{r}{4d\smo}$, we have that
\begin{align*}
    \bb E_{z_{<(t+1)}, X_{<(t+1)}, P_{<(t+1)}}\norm{\nabla F(w_t)}^2
    & \leq
    \frac{4 \bb E_{z_{<(t+1)}, X_{<(t+1)}, P_{<(t+1)}}[F(w_t) - F(w_{t+1}) | Q_t]\bb P(Q_t)}{\eta} \\
    &+ 2\smo\eta d\sigma^2 + \frac{16 d\eta \ind_{B < n} \lip^2 \smo}{Br} \\
    & + 4 \lip^2\left[1 + \frac{C}{\bb P\left[\abs{V} \geq \frac{C}{\lip}\right]\lip \sqrt{2\pi}}e^{-\frac{C^2}{2\lip^2}}\right]\;\bb P(\bar Q).
\end{align*}
Recall $Q_t$ is the event that clipping does not happen at iteration $t$ and $Q$ is the event that clipping does not happen for every iteration, hence $Q_t \cap Q = Q$. By the law of total probability and the assumption that $\abs{F(w_t) - F^*}\leq D$ for every $t$, we have that
\begin{align*}
    \bb E_{z_{<(t+1)}, X_{<(t+1)} P_{<(t+1)}}[F(w_t) - F(w_{t+1}) | Q_t]\bb P(Q_t)
    &=
    \bb E_{z_{<T}, X_{<T}, P_{<T}}[F(w_t) - F(w_{t+1}) | Q_t]\bb P(Q_t) \\
    =
    \bb E_{z_{<T}, X_{<T}, P_{<T}}&\Big[F(w_t) - F(w_{t+1}) \Big| Q_t \cap Q\Big]\bb P(Q_t \cap Q)  \\
    +
    \bb E_{z_{<T}, X_{<T}, P_{<T}}&\Big[F(w_t) - F(w_{t+1}) \Big| Q_t \cap \bar Q\Big]\bb P(Q_t \cap \bar Q) \\
    \leq
    \bb E_{z_{<T}, X_{<T}, P_{<T}}&[F(w_t) - F(w_{t+1}) | Q]\bb P(Q) + 2D \; \bb P(\bar Q).
\end{align*}
As a result, we have that
\begin{align*}
    \bb E_{z_{<(t+1)}, X_{<(t+1)}, P_{<(t+1)}}\norm{\nabla F(w_t)}^2
    & \leq
    \frac{4 \bb E_{z_{<T}, X_{<T}, P_{<T}}[F(w_t) - F(w_{t+1}) | Q]\bb P(Q)}{\eta}\\
    &+ 2\smo\eta d\sigma^2 + \frac{16 \eta d \ind_{B < n} \lip^2 \smo}{Br} \\
    & + \left(4\lip^2\left[1 + \frac{C}{\bb P\left[\abs{V} \geq \frac{C}{\lip}\right]\lip \sqrt{2\pi}}e^{-\frac{C^2}{2\lip^2}}\right] + \frac{8D}{\eta}\right)\;\bb P(\bar Q).
\end{align*}

Taking expectation with respect to all randomness, i.e., $\bb E=\bb E_{z_{<T}, \xi_{<T}, u_{<T}}$, summing up from $t=0$ to $T-1$, and dividing both sides by $T$, we have that
\begin{align*}
    \bb E\norm{\nabla F(w_\tau)}^2
    & =
    \frac{1}{T}\sum_{t=1}^{T} \bb E_{z_{<t}, \xi_{<t}, u_{<t}}\norm{\nabla F(w_t)}^2 \\
    & \leq
    \frac{4\,\bb E[F(w_0) - F(w_T) | Q] \bb P(Q)}{\eta T} + \frac{8\smo \eta T\, d\log(1/\delta)}{n^2\eps^2} C^2 + \frac{16 d\eta \ind_{B < n} \lip^2 \smo}{Br} \\
    & + \left(4 \lip^2\left[1 + \frac{C}{\bb P\left[\abs{V} \geq \frac{C}{\lip}\right]\lip \sqrt{2\pi}}e^{-\frac{C^2}{2\lip^2}}\right] + \frac{8D}{\eta}\right)\;\bb P(\bar Q) \\
    & \leq
    \frac{4\,\bb E[F(w_0) - F(w_T) | Q] \bb P(Q)}{\eta T} + \frac{8\smo \eta T\, d\log(1/\delta)}{n^2\eps^2} C^2 + \frac{16 d\eta \ind_{B < n} \lip^2 \smo}{Br} \\
    &+ \left(4 \lip^2\left[\bb P\left[\abs{V} \geq \frac{C}{\lip} \right] + \frac{C}{\lip\sqrt{2\pi}}e^{-\frac{C^2}{2\lip^2}}\right] + \frac{8D}{\eta}\bb P\left[\abs{V} \geq \frac{C}{\lip} \right]\right)TLnr \\
    & \leq
    \frac{4\,\bb E[F(w_0) - F(w_T) | Q] \bb P(Q)}{\eta T} + \frac{8\smo \eta T\, d\log(1/\delta)}{n^2\eps^2}C^2 \\
    & \quad + \frac{16 d\eta \ind_{B < n} \lip^2 \smo}{Br} + \lip^2 \left(4\left[2 + \frac{C}{\lip \sqrt{2\pi}}\right] + \frac{16D}{\lip^2 \eta}\right)TLnr e^{-\frac{C^2}{2\lip^2}}
\end{align*}
Considering the choice of parameters to be
\begin{align*}
    &\eta T = \frac{n \eps}{\smo \sqrt{2d\log(1/\delta)}}, \eta = \frac{r}{4d\smo}, T = \frac{2\sqrt{2d} n \eps}{r\sqrt{\log(1/\delta)}}, \\
    &C = \lip\sqrt{2\log\paren{\left(10 + \frac{64dD\smo}{r\lip^2}\right) \frac{n^3 \eps^2}{2\log(1/\delta)}}}, \\
    &B \geq \max\paren{\sqrt{\frac{rn}{2}}\paren{\frac{\log(1/\delta)}{2d}}^{1/4}, \frac{n \eps}{2\sqrt{2d\log(1/\delta)}}},
\end{align*}
we get that,
\begin{align*}
    \bb E\norm{\nabla F(w_\tau)}^2 \leq \left(\smo D + \paren{2 + \mathcal M + 2\log\paren{\left(10 + \frac{64dD\smo}{r\lip^2}\right) \frac{L n^3 \eps^2}{2\log(1/\delta)}}}\lip^2\right)\frac{4\sqrt{2d\log(1/\delta)}}{n\eps},
\end{align*}

where $\mathcal M = \frac{\log\paren{4\left[2 + \sqrt{\log\paren{\left(10 + \frac{64dD\smo}{r\lip^2}\right) \frac{n^3 \eps^2}{2\log(1/\delta)}}}\right] + \frac{64dD\smo}{r\lip^2}}}{\log{\left(10 + \frac{64dD\smo}{r\lip^2}\right)}} = \tilde{O}(1)$.
\end{proof}

\end{document}